\documentclass{article} 

\usepackage{preprint}
\usepackage{natbib}

\usepackage[hidelinks]{hyperref}
\usepackage{url}

\usepackage{amsmath, amsfonts, amsthm, amssymb}
\usepackage{notation}
\usepackage{wrapfig}
\usepackage{enumitem}
\usepackage{cleveref}

\usepackage{todonotes}

\usepackage{xcolor}
\usepackage{thmtools}
\usepackage{mdframed}
\mdfsetup{skipabove=5pt,skipbelow=2pt}

\usepackage[T1]{fontenc}
\usepackage{authblk}

\usepackage{subcaption}
\usepackage{wrapfig}

\usepackage{etoc}

\usepackage[ruled,vlined]{algorithm2e} 
\SetKwInput{KwInput}{input}
\SetKwInput{KwOutput}{output:}
\SetKw{KwInitialize}{init:}
\SetKw{KwRet}{return}
\DontPrintSemicolon
\SetKw{KwTo}{to}
\SetCommentSty{textit}
\SetKwComment{Comment}{\(\triangleright\)\ }{}

\declaretheoremstyle[mdframed={outerlinewidth=1pt,innertopmargin=5pt, innerbottommargin=2pt, backgroundcolor=gray!5, linecolor=gray!30}]{thmstyle}
\declaretheoremstyle[mdframed={outerlinewidth=1pt,innertopmargin=5pt, innerbottommargin=2pt, backgroundcolor=blue!5, linecolor=gray!30}]{assumptionstyle}

\declaretheorem[style=thmstyle]{theorem}
\declaretheorem[style=thmstyle]{proposition}
\declaretheorem[style=thmstyle]{corollary}
\declaretheorem[style=thmstyle]{lemma}
\declaretheorem[style=thmstyle]{definition}

\declaretheorem[style=remark]{remark}

\title{Distributionally robust linear regression \\ through the lens of adversarial training}

\author[1]{Elis Stefansson}
\author[1]{David Vävinggren}
\author[1,2]{Antônio H. Ribeiro}

\affil[1]{Uppsala University, Sweden}
\affil[2]{Science for Life Laboratory, Sweden}

\begin{document}

\maketitle
\etocdepthtag.toc{main} 

\begin{abstract}
Distributionally robust optimization (DRO) studies parameter estimation under uncertainty in the underlying probability distribution and has emerged as a principled framework for analyzing robustness and generalization. In particular, Wasserstein DRO, with distributional uncertainty induced by the Wasserstein distance, generalizes several popular regularizers. 
This paper studies Wasserstein DRO linear regression, unifying square-root Lasso and adversarial linear regression as important special cases. We prove that many properties of these two special cases carry over to this general method. In particular, we show (i) deterministic and non-asymptotic in-sample error bounds $O(n^{-1/2})$ in general and $O(n^{-1})$ under design matrix and sparsity conditions; (ii) insensitivity to the noise level, also known as the pivotal property; and (iii) solution equivalences for small and large ambiguity sets. The key proof step is to recast the method into a quadratic form, mimicking adversarial linear regression. We also show that the method can be solved efficiently, and we validate our findings through numerical~simulations.
\end{abstract}

\section{Introduction}

Distributionally robust optimization (DRO) studies parameter estimation under uncertainty in the probability distribution. Consider the traditional linear regression setting, where we have $n$ input-output samples $(\x_i,y_i)_{i=1}^n$ and we want to find a linear mapping between inputs $\x$ and outputs $y$, with squared error loss. The DRO formulation solves this problem by minimizing the expected loss with respect to a decision parameter $\param$ under the worst-case~distribution:
\begin{equation}
\label{dro}
    \inf_{\param} \sup_{\probQ \in \mathbb{B}_\delta(\empProb)} \E_{(\x,y) \sim \probQ}[(\param^\top \x - y)^2].
\end{equation}
Here, $\mathbb{B}_\delta(\empProb)$ is a ball of radius $\delta\geq0$ around the empirical distribution $\empProb$ capturing the set of distributions, also known as the ambiguity set. DRO serves as a principled framework for capturing \emph{distributional} uncertainty, being a natural tool for studying both generalization and out-of-distribution performance~\citep{shafiee_nash_2025,kuhn_distributionally_2025}.

In this paper, we  build on the broader relationship between robustness and regularization. Following the introduction of the regularization method Lasso~\citep{tibshirani_regression_1996}, several variants have been suggested to improve its statistical and practical properties. In particular, square-root Lasso was proposed as an alternative to Lasso that allows the regularization strength to be set without knowing the variance of the noise \citep{belloni_squareroot_2011}, a property sometimes called \emph{pivotal} estimation. The procedure has also been shown to be an instance of robust optimization \citep{xu_robust_2008}. More recently, with the growing interest in adversarial robustness, researchers found that adversarial linear regression has the same pivotal property~\citep{ribeiro_regularization_2023,xie2024high}.

One interesting observation is that both square-root Lasso and adversarial linear regression can be unified as instances of DRO. Indeed, letting the ambiguity ball $\mathbb{B}_\delta(\empProb)$ be induced by the $p$-Wasserstein distance (an optimal transport distance), there are close connections to several regularizers~\citep{kuhn_wasserstein_2019}. In particular, linear regression using the $2$-Wasserstein distance is equivalent to square-root Lasso  \citep{blanchet2019robust}, while $\infty$-Wasserstein DRO corresponds to adversarially trained linear regression (see, e.g., \citealt[Theorem 2]{zhang_short_2024}). This raise the natural question whether other methods in the $p$-Wasserstein DRO family share the same properties, such as being pivotal estimators.


\textbf{Contributions.}
In this paper, we study linear regression under square loss, where the ambiguity set is given by the $p$-Wasserstein distance with $2\leq p \leq \infty$. For short, we call this method \emph{Wasserstein DRO linear regression}. We show that many properties of the special cases square-root Lasso ($p=2$) and adversarially trained linear regression ($p=\infty$) carry over to this more general method, including the pivotal property, slow and fast error rates, and equivalences in the small and large $\delta$ regimes. More precisely, our contributions are:


\begin{itemize}[leftmargin=1.5em]
    \item We derive an equivalent form of the robust risk (the inner supremum of \eqref{dro}) for Wasserstein DRO linear regression (\Cref{equivalent_forms}). 
    \item We prove deterministic and non-asymptotic in-sample error bounds $O(n^{-1/2})$ in general and $O(n^{-1})$ under design matrix and sparsity conditions, generalizing results that were previously only known for square-root Lasso and adversarially trained linear regression  (\Cref{error-bounds}).
    \item We characterize the solution for small and large radii $\delta$ in \eqref{dro}, respectively, where small values correspond to minimum norm interpolation in the overparametrized setting, and large values to the zero solution being optimal~(\Cref{additional_properties}).
    \item We provide tailored numerical solvers for Wasserstein DRO linear regression (\Cref{efficient_solvers}) and conduct numerical experiments validating our findings (\Cref{numerical_experiments}).
\end{itemize}

The paper is structured as follows. Section \ref{related_work} covers related work. Section \ref{background} presents preliminaries. Sections 4, 5 and 6 provide analytic properties we described in the contributions above, followed by numerical solvers (Section \ref{efficient_solvers}), simulations  (Section \ref{numerical_experiments}) and conclusion (Section \ref{conclusion}). Our implementation and experiments are available at \href{https://github.com/elisst/dro}{\texttt{https://github.com/elisst/dro}}.

\section{Related work}\label{related_work}
DRO has a long history dating back to the pioneering work of \citet{scarf1957min}, and has recently received substantial interest in the machine learning community as a principled way to account for distributional uncertainty in the data; see, e.g., the recent survey \citep{kuhn_distributionally_2025}.

\textbf{Wasserstein DRO.} DRO with $p$-Wasserstein distance has been studied in both the classification and the regression settings, see the survey \citep{kuhn_wasserstein_2019} and the references therein. In particular, Wasserstein DRO linear regression (under various loss functions) has mainly been studied for the special cases $p=1$ and $p=2$, see \citep{blanchet2019robust,gao_wasserstein_2020,blanchet2021statistical,shafieezadeh2019regularization,kuhn_wasserstein_2019,aolaritei2026wasserstein}. Notably, \citet{aolaritei2026wasserstein} seek asymptotic statistical guarantees for a class of linear prediction problems in  the high-dimensional regime with  $p=1$ and $p=2$. We instead provide non-asymptotic statistical guarantees for Wasserstein DRO linear regression (with square loss) uniformly over the range $2 \leq p \leq \infty$.\footnote{Here, $p<2$ is excluded since the robust risk in \eqref{dro} then becomes infinite for square loss. Thus, the natural range to consider is $2 \leq p \leq \infty$.}

\citet{wu2026generalization} consider generalization bounds for a general learning objective, which could be applied to linear regression. However, their high-probability bound scales as $O(n^{-1/(2p)})$ with an improved ($p$-independent) slow rate of $O(n^{-1/2})$ that is not applicable to square loss \citep[Assumption 1]{wu2026generalization}. There are also other generic high-probability generalization bounds that, e.g., assume Lipschitz loss \citep{shafieezadeh2019regularization,an2021generalization,gao2023finite} or a compact domain \citep{le2025universal}, neither of which holds for square loss. Additionally, these generic bounds (derived from concentration-based methods) typically have no explicit dependence on the noise. We instead derive \emph{deterministic} error bounds with \emph{explicit} dependence on the noise, inspired by \citet{belloni_squareroot_2011,ribeiro_regularization_2023,xie2024high} and \citet[Chapter 7]{wainwright_high-dimensional_2019}, that with high probability enjoy both slow rates $O(n^{-1/2})$ and fast rates $O(n^{-1})$ \emph{uniformly} in $p$.


\textbf{Connection between DRO and regularization.} Several well-known regularization techniques in machine learning have found connections with DRO, such as square-root Lasso \citep{xu_robust_2008,blanchet2019robust} with higher-order generalizations \citep{olea2022out}, regularized logistic regression and support vector machines \citep{blanchet2019robust}, ridge regression \citep{shafieezadeh2019regularization,li2022tikhonov}, adversarial training \citep{gao_wasserstein_2020,pydi2021many}, as well as close correspondences to total variation, Lipschitz-variation and gradient variation \citep{gao_wasserstein_2020} as well as higher-order variation \citep{shafiee2026nash}. In particular, square-root Lasso is equivalent to Wasserstein DRO linear regression with square loss and $p=2$ \citep{blanchet2019robust}, while the $p=\infty$ case coincides with adversarial linear regression, since adversarial training and $\infty$-Wasserstein DRO are equivalent (under mild conditions); see \citet{zhang_short_2024} for a general result. Our work builds upon this research direction by unifying these two instances as special cases of the more general Wasserstein DRO linear regression with square loss ($2\leq p\leq\infty$), where \emph{we extend fundamental properties such as error rates and the pivotal property to this unifying~method.}


\textbf{Tractable reformulations.} There has been substantial work to make DRO tractable by reformulating the DRO objective  as a finite-dimensional optimization problem, commonly via a dual representation of the inner supremum \citep{shafieezadeh2019regularization,shafiee2026nash}, see \citet{zhang_short_2024} for a general result. In particular, using this dual representation, \citet[Proposition 2 and 3]{shafiee2026nash} and \citet[Theorem 7.20]{kuhn_distributionally_2025} both show that the DRO objective is equivalent to a finite convex program, under assumption that the loss is a (finite maximum of) concave function(s). However, such assumptions do not hold in our case due to the square loss. 

The work \citet{shafiee2026nash} also considers linear prediction models and shows that the dual representation then admits a lower-dimensional form, where an inner supremum (cf. \Cref{optimal_transport_dro_dual}) is recast into an expectation over a (possibly nonconvex) univariate maximization (see \citealt[Theorem 4]{shafiee2026nash}). In this paper, we note that for Wasserstein DRO linear regression (a special case of linear prediction models), this one-dimensional representation becomes unimodal, which we use in the numerical solvers. Furthermore, \citet{kuhn_wasserstein_2019} present convex forms for $p$-Wasserstein linear regression with square loss and $p=2$ (and $p=1$ for other loss functions), whereas we consider the whole range $2 \leq p \leq~\infty$. Finally, we note that our dual form (\Cref{th_supremum_robust_risk}) naturally generalize adversarial linear regression \citep[Proposition 1]{ribeiro_regularization_2023} by coupling the radius constraints, and also identifies a case in which robust optimization forms coincides with DRO (see, e.g., \citealt[Corollary 2(iii)]{gao2023distributionally} and \citealt[Proposition 2] {an2021generalization} for approximate~results).



\section{Background}\label{background}
In this section, we describe the setup and relevant results in DRO and adversarial linear regression.

\textbf{Wasserstein distance.} Let $\mathcal{P}(\mathcal{Z})$ denote the set of probability distributions with support $\mathcal{Z}$. One common choice of distance between such probability distributions is the \emph{$p$-Wasserstein distance} $W_p$, where $1 \leq p \leq \infty$ is the order of this distance. More precisely, the $p$-Wasserstein distance $W_p$ with $1 \leq p<\infty$ is given by
\begin{align}\label{eq_Wasserstein_p_not_inf}
W_p(\probP, \empProb) := \left ( \inf_{\pi \in \Pi(\probP, \empProb)} \mathbb{E}_{(\z, \z') \sim \pi}[\| \z-\z' \|^p] \right )^{1/p},
\end{align}
where $\Pi(\probP, \empProb)$ is the set of all distributions in $\mathcal{P}(\mathcal{Z}) \times \mathcal{P}(\mathcal{Z})$ with marginals $\probP$ and $\empProb$, respectively, and $\| \cdot \|$ is a norm. Intuitively, $\| \z-\z' \|^p$ in \eqref{eq_Wasserstein_p_not_inf} captures the cost of moving probability mass from $\z$ to $\z'$, where the infimum picks the strategy $\pi$ with the least total cost for moving $\probP$ to $\empProb$. The normalization with the $p$-th root is to ensure that $W_p$ is a metric~\citep{kuhn_wasserstein_2019}. Moreover, for $p=\infty$, we have $W_{\infty}(\probP, \empProb) := \inf_{\pi \in \Pi(\probP, \empProb)} \text{ess sup}_{(\z, \z') \sim \pi}[\|\z-\z'\|]$.

\textbf{Strong duality.}
Consider the Wasserstein DRO problem
\[
\inf_{\param} \sup_{\probQ \in \mathbb{B}_\delta(\empProb)} \E_{\z \sim \probQ}[h_{\param}(\z)]\]
with objective $h_{\param}: \mathcal{Z} \rightarrow \R$ and ambiguity set $\mathbb{B}_\delta(\empProb) = \{\mathbb{Q}  \in \mathcal{P}(\mathcal{Z}) : W_p(\mathbb{Q}, \empProb) \leq \delta \}$ induced by the $p$-Wasserstein distance centered around the empirical distribution $\empProb$, where the empirical distribution is $\empProb := \frac{1}{n} \sum_{i=1}^n \delta_{\z_i}$ for $\delta_{\z_i}$ the Dirac point mass at data sample $\z_i \in \mathcal{Z}$. 

Here, we will denote by $V_\delta$ the inner supremum objective function. In particular, with the given definition of $\mathbb{B}_\delta(\empProb)$, we can write: 
\begin{align}\label{eq_normalised_robust_risk}
{V}_\delta(\param) := \sup_{ \probQ} \{\E_{\z \sim \probQ}[h_{\param}(\z)]: W_p(\probQ, \empProb) \le \delta \},
\end{align}
which we will call the \emph{robust risk}. Using strong duality, we have the following result:
\begin{proposition}[Strong duality]
\label{optimal_transport_dro_dual}
Let $h_{\param}: \mathcal{Z} \rightarrow \R$ be a measurable function with $\E_{\empProb}[h_{\param}] > -\infty$. Then the robust risk in \eqref{eq_normalised_robust_risk} equals:
\begin{equation}\label{eq:robust_risk_intermediate_form}
{V}_{\delta}(\param)\ = \min_{\lambda\ge 0} \left(\lambda \delta^p + \E_{\z \sim \empProb}\Big[\sup_{\z'\in \domz} \big(h_{\param}(\z') - \lambda \|\z'-\z\|^p\big)\Big] \right).
\end{equation}
\end{proposition}
\Cref{optimal_transport_dro_dual} is a straightforward application of the general result \citet[Theorem 1]{zhang_short_2024} and the first step to our equivalent form for linear regression (\Cref{th_supremum_robust_risk}) in Section \ref{equivalent_forms}.

\textbf{Wasserstein DRO linear regression.}
We focus on $p$-Wasserstein DRO linear regression. That is, $\z_i = (\x_i,y_i) \in \mathcal{Z}=\R^{d+1}$ and $h_{\param}(\z) = (\param^\top\x - y)^2$.  We restrict our attention to the norm   $\|\z_i'-\z_i \| = \|\x_i'-\x_i \| + \infty |y_i'-y_i|$. Here, $\infty |y_i'-y_i|$ is shorthand for $+\infty$ if $y_i' \neq y_i$ and zero otherwise. These restrictions result in a close correspondence to square-root Lasso and adversarial linear~regression.

\begin{remark}\label{remark_p_less_2_blows_up}
We exclude $1 \leq p <2$ since in this case the robust risk $V_{\delta}(\param)$ becomes infinite. The reason is that the squared loss increases faster than the penalty for moving probability mass, resulting in an unbounded robust risk. See the appendix for a detailed derivation.
\end{remark}

\textbf{Adversarial linear regression.} Our work also considers connections with adversarial linear regression. Adversarial linear regression considers linear regression where each input $\x_i \in \R^d$ is adversarially perturbed by $\Delta \x_i$ within a \emph{fixed} budget $\|\Delta \x_i \| \leq \delta$ (uniform over all inputs) to form the worst-case risk:\footnote{Note that with no budget $\delta=0$, this becomes traditional linear regression with square loss.}
\begin{equation}\label{eq_robust_risk_adv_lin_reg}
    V^{\textrm{adv}}_{\delta}(\param) := \frac{1}{n} \sum_{i=1}^n \sup_{\|\Delta \x_i\| \leq \delta}((\x_i+\Delta \x_i)^{\top}\param-y_i)^2 =
    \frac{1}{n} \sum_{i=1}^n (|\x_i^{\top}\param-y_i|+\delta \|\param\|_*)^2,
\end{equation}
where the last equality follows by rewriting the supremum inside the sum using a simple dual norm argument; see, e.g., Proposition 1 in \citet{ribeiro_regularization_2023}. In Section \ref{equivalent_forms}, we generalize this form to Wasserstein DRO linear regression (\Cref{th_supremum_robust_risk}), a key step in establishing the analytical properties.

\section{Equivalent form}\label{equivalent_forms}
In this section, we present an equivalent form of the robust risk ${V}_\delta(\param)$, called the \emph{robust quadratic form}. All analytic properties of Wasserstein DRO linear regression shown in subsequent sections are based on this form.

\begin{theorem}[Robust quadratic form]\label{th_supremum_robust_risk}
The robust risk of Wasserstein DRO linear regression ($2 \leq p \leq \infty$) equals
\begin{equation}\label{supremum_form_linear_regression}
V_{\delta}(\param) = \sup_{t_i \geq 0, \; \|\vv{t}\|_p \leq n^{1/p} \delta} \frac{1}{n} \sum_{i=1}^n (|\x_i^{\top}\param-y_i|+t_i \|\param\|_*)^2.
\end{equation}
\end{theorem}

Crucially, the right-hand side of \eqref{supremum_form_linear_regression} can naturally be interpreted as an adversarial linear regression with a \emph{variable} perturbation budget for each sample. Indeed, for $p=\infty$, we have that \eqref{supremum_form_linear_regression} coincides with \eqref{eq_robust_risk_adv_lin_reg} since the optimal outer constraint becomes $\|\vv{t}\|_{\infty} = \delta$ (optimum is attained on the boundary) yielding $\|\Delta \x_i\| \leq \delta$, i.e., a \emph{fixed} perturbation budget $\delta$ for each input sample $\x_i$. In the general $2 \leq p \leq \infty$ case, these perturbation budgets $t_i$ can \emph{vary} between samples, as long as the overall budget constraint $\|\vv{t}\|_p \leq n^{1/p} \delta$ is fulfilled. Thus, \eqref{supremum_form_linear_regression} \emph{naturally generalizes the objective of adversarial linear regression}. This is the key step to show analytical properties, since it allows us to extend proofs from the adversarial setting.

The proof of \Cref{th_supremum_robust_risk} starts by recasting the supremum in \eqref{eq:robust_risk_intermediate_form} as a one-dimensional concave supremum from which \eqref{supremum_form_linear_regression} is shown using a change of variables and strong duality; see the appendix for details.
Finally, \Cref{th_supremum_robust_risk} yields a simple proof not only that $p=\infty$ coincides with adversarial linear regression, but also that $p=2$ corresponds to square-root Lasso \citep{xu_robust_2008,blanchet2019robust}, by simply computing the robust risk. We summarize these findings with the following corollary (using $V_{p,\delta}(\param)$ to temporarily stress dependence on $p$):


\begin{corollary}[Special cases]\label{th_special_cases}
Wasserstein DRO linear regression with $p=2$ has robust risk $V_{2,\delta}(\param) = \left ( \sqrt{\frac{1}{n}\sum_{i=1}^n (\x_i^T \param - y_i)^2}+\delta \|\param\|_* \right)^2$, and thus corresponds to square-root Lasso by taking $\| \cdot \|=\| \cdot \|_{\infty}$, 
while $p = \infty$ corresponds to adversarial linear regression with the continuous limit
$V_{\infty,\delta}(\param) = \lim_{p \rightarrow \infty} V_{p,\delta}(\param) =
    \sum_{i=1}^n \sup_{\|\vv{v}_i\| \leq \delta} ( (\x_i+\vv{v}_i)^T \param - y_i )^2$.
\end{corollary}

\begin{remark}
Weaker versions of \Cref{th_supremum_robust_risk} are described in other papers.
The right-hand side of \eqref{supremum_form_linear_regression} is shown to lower-bound the robust risk~\cite[Theorem 6]{kuhn_wasserstein_2019}, but that work does not explicitly show when equality holds. For $p=\infty$, \citet[Theorem 2]{zhang_short_2024} proves this result in general, while \citet[Proposition 6.16]{kuhn_distributionally_2025} shows that equality holds and a saddle-point exists, but only for compact sets. We show that equality in fact holds for all $2 \leq p \leq \infty$ for linear regression. See also \citet[Corollary 2(iii)]{gao2023distributionally} and \citet[Proposition 2]{an2021generalization} for results that only hold~approximately.
\end{remark}

\section{Error bounds}
\label{error-bounds}

In this section, we derive error bounds for Wasserstein DRO linear regression (with $2 \leq p \leq \infty$). We consider the following data-generating model  $y_i=\x_i^{\top} \param^*+\varepsilon_i$ with true parameter $\param^* \in \R^d$ and noise $\varepsilon_i \in \R$.  More precisely, the quantity of interest is the average in-sample error $\frac{1}{n}\|X \widehat{\Delta}\|_2^2$, where $X$ is the $n \times d$ data matrix with row $i$ given by $\x_i$, and $\widehat{\Delta} := \widehat{\param}-\param^*$ is the difference between the estimated parameter $\widehat{\param}$ and the true parameter $\param^*$, where $\widehat{\param}$ is the minimizer of $V_\delta(\param)$. The aim is to upper bound $\frac{1}{n}\|X \widehat{\Delta}\|_2^2$ as a function of $\delta$, which we can then use to find good decay rates by tuning $\delta = \delta(n)$ appropriately.\footnote{The quantity $\frac{1}{n}\|X \widehat{\Delta}\|_2^2$ is also called the excess risk because it equals the difference between the expected risk
$\E_{\varepsilon} [\frac{1}{n}  \sum_{i=1}^n(y_i - \x_i^{\top}  \widehat\param )^2]$ and the Bayes optimal risk $\E_{\varepsilon} [\frac{1}{n} \sum_{i=1}^n(y_i - \x_i^{\top}  \param^* )^2]$ (\citet[Prop. 3.3]{bach_learning_2024}).}

We derive a slow error decay rate $O(1/\sqrt{n})$ (\Cref{improved_in_sample_error_decay}) without any additional assumptions on the data. In Section \ref{fast_rate}, we then improve this to a fast rate $O(1/n)$ (Theorem \ref{thm_fast_rate_main_theorem_v2}) by assuming a restricted eigenvalue condition and focusing on the infinity norm $\| \cdot \| = \| \cdot \|_\infty$. Moreover, for both the slow rate and the fast rate, the pivotal property (that is, tuning $\delta = \delta(n)$ to achieve the rate is independent of the noise level) comes out for free.

\subsection{Slow  rate of \texorpdfstring{$O(1/\sqrt{n})$}{O(1/√n)}}

\label{slow_rate}
To show the desired error rate, we first show the following bound on $\frac{1}{n}\|X \widehat{\Delta}\|_2^2$:

\begin{lemma}\label{th_in_sample_error_prestage_improved}
    The average in-sample error is bounded~by
    \begin{equation}
    \frac{1}{n}\|X \widehat{\Delta}\|_2^2 \leq \frac{2}{n} \vv{\varepsilon}^{\top} X \widehat{\Delta}+ 2 \delta\frac{\|\vv{\varepsilon}\|_q}{n^{1/q}} \|\param^*\|_*+\delta^2 \|\param^*\|_*^2.
    \end{equation}
\end{lemma}



\Cref{th_in_sample_error_prestage_improved} is a first step toward obtaining sharp in-sample error upper bounds, where the important quantity to bound is $\frac{2}{n} \vv{\varepsilon}^{\top} X \widehat{\Delta}$. Crucially, $\frac{2}{n} \vv{\varepsilon}^{\top} X \widehat{\Delta} \leq \frac{2}{n} \|X^{\top} \vv{\varepsilon}\| \|\widehat{\Delta}\|_*$. Thus, provided $\|\widehat{\Delta}\|_*$ is bounded (true if $\|\widehat{\param}\|_*$ is bounded) and $\frac{2}{n} \|X^{\top} \vv{\varepsilon}\|$ decreases with $n$ (typical for, e.g., Gaussian noise $\vv{\varepsilon}$), we may obtain decreasing upper bounds as long as $\delta$ decreases with $n$. The next result shows that $\|\widehat{\param}\|_*$ can be bounded by $\|\param^*\|_*$:

\begin{lemma}\label{th_controlled_dual_norm_improved_v3}
    If $\delta > 2\frac{\|X^{\top} \vv{\varepsilon}\|}{n^{1/p}\|\vv{\varepsilon}\|_q}$, then
    \begin{equation}\label{eq1_controlled_dual_norm_improved_v3}
        \|\widehat{\param}\|_* \leq \left ( 3+\frac{n^{1/q} \delta}{\| \vv{\varepsilon}\|_q} \|\param^*\|_* \right ) \|\param^*\|_*
    \end{equation}
\end{lemma}


Importantly, \Cref{th_controlled_dual_norm_improved_v3} combined with \Cref{th_in_sample_error_prestage_improved} enables us to prove a slow rate $O(n^{-1/2})$, with the slow rate of square-root Lasso ($p=2$) and adversarial training ($p=\infty$) as special cases:


\begin{theorem}\label{improved_in_sample_error_decay}
    Assume $\delta > \bar{\delta} := 2\frac{\|X^{\top} \vv{\varepsilon}\|}{n^{1/p}\|\vv{\varepsilon}\|_q}$. Then
    \begin{equation}\label{improved_in_sample_error_decay_eq}
        \frac{1}{n}\|X \widehat{\Delta}\|_2^2 \leq 6 \frac{\|\vv{\varepsilon}\|_q}{n^{1/q}}\delta\|\param^*\|_*+2\delta^2 \|\param^*\|_*^2.
    \end{equation}
    In particular, let $\vv{\varepsilon} \sim N(0,\sigma^2 I_n)$ and assume each entry of $X$ is bounded by $M>0$. Then, with high probability (independent of the Wasserstein parameter $p$):
    \begin{equation*}
        \frac{1}{n}\|X \widehat{\Delta}\|_2^2 \in O(\delta).
    \end{equation*}
    That is, the error rate coincides with $\delta$ (as long as $\delta> \bar{\delta}$). In particular, let $c>0$ be the norm-equivalence constant between $\|\cdot \|_\infty$ and $\| \cdot \|$, that is, $\|\cdot \|_\infty \leq c\| \cdot \|$. Then, by setting $\delta = K M \sqrt{ \frac{\log(d/\gamma)}{n}}$ with $K > 2c \sqrt{\pi}$, we get with probability greater than $1-4\gamma$:
    \begin{equation}\label{eq_improved_in_sample_error_decay_rate}
    \frac{1}{n}\|X \widehat{\Delta}\|_2^2 \in O(n^{-1/2}).
    \end{equation}
\end{theorem}

\Cref{improved_in_sample_error_decay} shows that $p$-Wasserstein linear regression can achieve the rate $O(1/\sqrt{n})$, where we recover the previously obtained rates for square-root Lasso \citep{belloni_squareroot_2011} ($p=2$) and adversarial linear regression \citep{ribeiro_regularization_2023} ($p =\infty$).

\textbf{Pivotal property.} The tuning $\delta = K M \sqrt{ \frac{\log(d/\gamma)}{n}}$ in \Cref{improved_in_sample_error_decay} is independent of the noise level $\sigma$ of $\vv{\varepsilon} \sim N(0,\sigma^2 I)$; hence, \emph{we do not need to estimate $\sigma$}. This independence is known as the \emph{pivotal property}. Previously, both square-root Lasso ($p=2$) and adversarial linear regression ($p = \infty$) were known to be pivotal \citep{belloni_squareroot_2011,ribeiro_regularization_2023}. Here, we show that the same holds for the more general $p$-Wasserstein linear regression with $2 \leq p \leq \infty$, suggesting that the pivotal property is tightly linked to Wasserstein robustification. 

\begin{proof}[Sketch of proof of \Cref{improved_in_sample_error_decay}]
We bound $\frac{2}{n} \vv{\varepsilon}^{\top} X \widehat{\Delta}$ as $\frac{2}{n} \vv{\varepsilon}^{\top} X \widehat{\Delta} \leq \frac{2}{n} \|X^{\top} \vv{\varepsilon}\| \|\widehat{\Delta}\|_* \leq \frac{2}{n} \|X^{\top} \vv{\varepsilon}\|(\|\widehat{\param}\|_*+\|\param^*\|_*)$ and then use \Cref{th_controlled_dual_norm_improved_v3} with $\delta n^{1/p}\|\varepsilon\|_q >  2\|X^{\top} \vv{\varepsilon}\|$ to get
    $\frac{2}{n} \vv{\varepsilon}^{\top} X \widehat{\Delta} \leq  4\frac{\|\vv{\varepsilon}\|_q}{n^{1/q}} \delta \|\param^*\|_* + \delta^2 \|\param^*\|_*^2$,
which, when inserted into \Cref{th_in_sample_error_prestage_improved}, yields \eqref{improved_in_sample_error_decay_eq}. Moreover, the rates are obtained by standard high-probability bounds for Gaussians. For a full proof, see the appendix.
\end{proof}



\subsection{Fast rate of \texorpdfstring{$O(1/n)$}{O(1/n)} under the restricted eigenvalue condition}\label{fast_rate}
In this section, we show the fast in-sample error rate $O(1/n)$ by assuming a restricted eigenvalue condition and focusing our attention to the infinity norm $\| \cdot \| = \| \cdot \|_\infty$. The restricted eigenvalue condition~reads \citep{Hastie2015_book}:
\begin{definition}[Restricted eigenvalue condition \citep{Hastie2015_book}]\label{def_restricted_eigenvalue_condition_v2}
The matrix $X \in \R^{n \times d}$ satisfies the restricted eigenvalue (RE) condition if there exists a positive constant $\kappa = \kappa(s,l)>0$ such that 
\begin{equation*}
\kappa \leq \min \left \{ \frac{\|X\vv{v}\|_2}{\sqrt{n}\| \vv{v}\|_2}: |S| \leq s, \vv{v} \in \R^d \backslash \{0\}, \|\vv{v}_{S^c}\| \leq l \|\vv{v}_{S}\| \right \}
\end{equation*}
where $S$ is some subset of $\{1,2,\dots,d\}$.
\end{definition}
Let $\mathrm{RE}(s,l)$ denote the set of matrices $X$ that satisfy the RE condition. We have the fast rate result:

\begin{theorem}\label{thm_fast_rate_main_theorem_v2}
Let $\param^*$ have $s$ nonzero entries indexed by $S \subseteq \{1,\dots,d\}$ with $s = |S|$, assume $X$ is in $\mathrm{RE}(s,3)$ with positive constant $\kappa = \kappa(s,3)$ (same for all $n$), and consider $\| \cdot \| = \| \cdot \|_\infty$. Then with $\delta > \bar{\delta} := 2\frac{\|X^{\top} \vv{\varepsilon}\|}{n^{1/p} \|\vv{\varepsilon}\|_q}$, we have
\begin{equation}\label{eq_fast_rate_main_theorem_v2_bound}
    \frac{1}{n}  \| X \widehat{\Delta} \|_2^2 \leq \delta^2 \cdot \max \left \{ 4 C^2 \|\param^*\|_1^2, \frac{16 B^2 s^2}{\kappa^2} \right \}
\end{equation}
where $B = 3 \frac{\|\vv{\varepsilon}\|_q}{n^{1/q}} + (2MC+\delta C+\delta) \|\param^*\|_*$ and $C = (  3+\frac{n^{1/q} \delta}{\| \vv{\varepsilon}\|_q} \|\param^*\|_* )$ and each entry of $X$ is bounded by $M>0$. In particular, for $\vv{\varepsilon} \sim N(0,\sigma^2 I_n)$, we have with high probability that $B$ and $C$ are bounded by constants and therefore (independent of the Wasserstein parameter $p$)
\begin{equation}
    \frac{1}{n} \|X \widehat{\Delta} \|_2^2 \in O \left (\delta^2 \right).
\end{equation}
Thus, setting $\delta$ as in \Cref{improved_in_sample_error_decay}, we get the fast rate with probability greater than $1-5\gamma$:
\begin{equation}
    \frac{1}{n} \|X \widehat{\Delta} \|_2^2 \in O\left ( 1/n \right ).
\end{equation}
\end{theorem}


\Cref{thm_fast_rate_main_theorem_v2} shows that the general $p$-Wasserstein linear regression ($2 \leq p \leq \infty$) can achieve the rate $O(1/n)$, similar to previous results obtained for square-root Lasso \citep{belloni_squareroot_2011} ($p=2$) and for adversarial linear regression \citep{xie2024high} ($p = \infty$). Moreover, just as in the slow rate, the tuning scheme $\delta = K M \sqrt{ \frac{\log(d/\gamma)}{n}}$ makes the method~pivotal.



\section{Small and large ambiguity balls}\label{additional_properties}
Here, we characterize the minimizer $\widehat{\param}$ of $V_\delta(\param)$ when the radius $\delta$ of the ambiguity set is small and large, respectively. These results generalize conditions from adversarial linear regression \citep[Theorem 1 and Proposition 3]{ribeiro_regularization_2023} to the general $p$-Wasserstein linear regression~($2 \leq~ p \leq \infty$). Importantly, these results also apply to square-root Lasso ($p=2$), which is novel to the best of our knowledge. We start with the large regime.

\textbf{Large $\delta$ regime.}
The first result shows that the zero solution $\param = 0$ minimizes $V_\delta(\param)$ for a sufficiently large radius $\delta \geq \delta_{L}$. In other words, for large ambiguity sets, the optimal solution is to be maximally conservative.

\begin{theorem}[Large $\delta$ regime]\label{th_when_beta_zero}
The zero solution $\param = 0$ minimizes $V_\delta(\param)$ if and only if
\begin{equation*}
    \delta \geq \delta_{L} := \frac{\|X^{\top} \y \|}{n^{1/p} \|\y\|_q}.
\end{equation*}
\end{theorem}

\textbf{Small $\delta$ regime.} For the small $\delta$ regime, we follow \citet{ribeiro_regularization_2023} and consider the overparametrized setting where we assume $X \in \R^{n \times d}$ to have full row rank and $n < d$. In this setting, there are multiple solutions to the original least-squares problem $X\param = y$, where a natural solution is to take the minimal $\| \cdot \|_N$-norm interpolator $\bar{\param}$ that solves $\min_{X\param = y} \|\param\|_N$ for a given norm~${\| \cdot \|_N}$. The next result shows that for a small enough radius $\delta \leq \delta_S$, the minimal $\| \cdot \|_*$-norm interpolator minimizes~$V_\delta(\param)$:

\begin{theorem}[Small $\delta$ regime]\label{th_when_beta_minimum_norm}
Let $X$ be of full row rank and $Q = \argmax_{\|X^T \vv{\alpha}\| \leq 1} \vv{\alpha}^{\top} \y$. Then the minimum $\|\cdot\|_*$-norm interpolator $\bar{\param}$ minimizes $V_\delta(\param)$ if and only if
\begin{equation*}
    \delta \leq \delta_S := 
    \begin{cases}
    \frac{1}{n\min_{\vv{\alpha}  \in Q}\|\vv{\alpha}\|_{\infty}} & \text{for $2<p\leq \infty$}\\
    \frac{1}{\sqrt{n}\min_{\vv{\alpha} \in Q}\|\vv{\alpha}\|_{2}} & \text{for $p=2$}
    \end{cases}
\end{equation*}
\end{theorem}
Note first that the set $\{\vv{\alpha}: \|X^{\top} \vv{\alpha}\| \leq 1\}$ is compact, and hence both $\min_{\vv{\alpha}  \in Q}\|\vv{\alpha}\|_{\infty}$ and $\min_{\vv{\alpha}  \in Q}\|\vv{\alpha}\|_{2}$ are finite. Thus, \emph{interpolation always occurs} for small enough $\delta$. Moreover, the interpolation threshold $\delta_S$ \emph{is the same for all} $2<p\leq \infty$ while it is \emph{higher} for $p=2$:
\begin{equation*}
    \frac{1}{\sqrt{n}\min_{\vv{\alpha}  \in Q}\|\vv{\alpha}\|_{2}} \geq \frac{1}{n\min_{\vv{\alpha}  \in Q}\|\vv{\alpha}\|_{\infty}},
\end{equation*}
by the norm inequality $\|\vv{\alpha}\|_2 \leq \sqrt{n} \|\vv{\alpha}\|_\infty$. In fact, this inequality is typically strict since the norm inequality is typically strict (see \Cref{remark_computing_delta_S} in the appendix: for $\| \cdot \| = \| \cdot \|_2$, the set $Q$ is a singleton $\bar{\vv{\alpha}}$, and varying $X$ and $\y$ generally gives different $\infty$-norms and $2$-norms). That is, \emph{interpolation occurs sooner} for square-root Lasso ($p=2$) than for $2<p\leq \infty$ (e.g., adversarial linear regression). We also stress that $\delta_S$ is easy to compute via a convex program, see \Cref{remark_computing_delta_S} in the appendix.

\section{Numerical solvers}\label{efficient_solvers}

We provide two solvers tailored specific for Wasserstein DRO linear regression. The first solver, called the \emph{saddle-point solver}, considers a novel saddle-point formulation of the robust risk $\min_{\param} n V_{\delta}(\param) = \min_{\param} \sup_{\vv{\gamma} \geq 0} H(\param,\vv{\gamma})$ with convex-concave function $H(\param,\vv{\gamma})$ on the form:
\begin{equation*}
H(\param,\vv{\gamma}) :=
 n^{1/p} \delta \|\param\|_* \|\vv{\gamma}\|_1^{1/q}+\sum_{i=1}^n \gamma^{1/q}_i |\x_i^T \param - y_i | - \gamma_i^{2/q}/4.
\end{equation*}
See appendix for a derivation. We solve this saddle problem using disciplined saddle programming \citep{schiele2024tmlr-disciplined}, an extension of disciplined convex programming \citep{diamond_cvxpy_2016}. 

\begin{wrapfigure}[13]{r}{0.52\textwidth}\vspace{-13pt}
  \centering
  \includegraphics[width=0.48\textwidth]{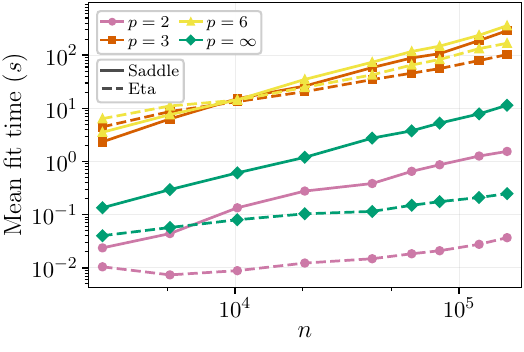}
  \vspace{-5pt}
  \caption{Mean computation times over 10 repetitions for the saddle-point solver (solid) and the $\eta$-trick solver (dashed) in the fast-rate experiment.}
  \label{fig_eta_saddle_comparison}
\end{wrapfigure}
The second solver, called the \emph{$\eta$-trick solver}, uses what is sometimes referred to as the $\eta$-trick~\citep{Itrick019} to iteratively alternative between solving a weighted ridge regression problem and updating the weights in closed-form. More precisely, we fix $\lVert  \cdot \rVert = \lVert  \cdot \rVert_{\infty}$ and rewrite $\|\param\|_* = \|\param\|_1$ in \eqref{supremum_form_linear_regression} as a minimum of a weighted quadratic form, with weights $\boldeta$. This becomes $V_{\delta}(\param) = \inf_{\boldeta} \Phi(\param,\vv{t}^{\star}, \boldeta)$ for a function $\Phi$ that is jointly convex in $(\param, \boldeta)$, and $\vv{t}^{\star} = \vv{t}^{\star}(\param)$ attains the supremum in \eqref{supremum_form_linear_regression}. We sequentially update these three variables repeatedly until convergence:
\begin{enumerate}[leftmargin=1.5em] \vspace{-6pt}
    \item Given $\param$, we update $\vv{t}^{\star}$;
    \item Given $\vv{t}^{\star}$ and $\param$, we compute the next $\boldeta$;
    \item Given new $\vv{t}^{\star},\boldeta$, we compute the next $\param$, and so on.
\end{enumerate}\vspace{-6pt}
Crucially, all these subproblems can be solved fast: the ${t}^{\star}$-update can be solved by a scalar formulation, the $\boldeta$-update has a closed-form solution and the $\param$-update is weighted ridge regression. The idea follows closely~\citep{pmlr-v258-ribeiro25a} proposed for adversarial linear training. We provide details about implementation and convergence in the appendix. The result is the improved speed we show in the numerical experiments. 

\begin{figure}[t]
\centering
\begin{subfigure}{.47\textwidth}
    \centering
    \includegraphics[width=1\linewidth]{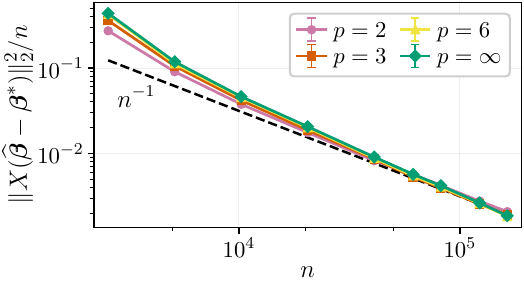}
\end{subfigure}
\quad
\begin{subfigure}{.47\textwidth}
    \centering
    \includegraphics[width=1\linewidth]{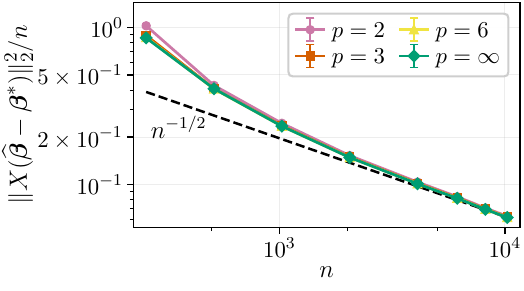}
\end{subfigure}
\caption{Fast rates (left) and slow rates (right) for $p \in \{2,3,6,\infty\}$. Both show a clear convergence toward the predicted rates $O(n^{-1})$ and $O(n^{-1/2})$, respectively. The curves are mean over 10 runs with $\pm1$ standard deviation bands.}
\label{fig_fast_and_slow_rate}
\end{figure}

\section{Numerical experiments}\label{numerical_experiments}

This section showcases our findings through simulations. We start with the fast rate $O(n^{-1})$ and the slow rate $O(n^{-1/2})$. Then, we illustrate the small and large $\delta$ regimes. We fix $\|\cdot \| = \| \cdot \|_{\infty}$, and let $p$ vary as $p\in\{2,3,6,\infty\}$. For the rates, we use 10 trials for each $(n,p)$ with $\pm1$ standard deviation bands (although, so tight they are barely visible). All simulations were run on a MacBook Air M3 in the matter of hours. For Figure~\ref{fig_fast_and_slow_rate} and Figure~\ref{fig_paper_extreme_delta}, we used the saddle-point solver (see \Cref{efficient_solvers}). In Figure~\ref{fig_eta_saddle_comparison} we compare the performance of the solvers. See the appendix for additional~simulations.

\begin{wrapfigure}{tr}{0.54\textwidth}
  \centering
  \includegraphics[width=0.53\textwidth]{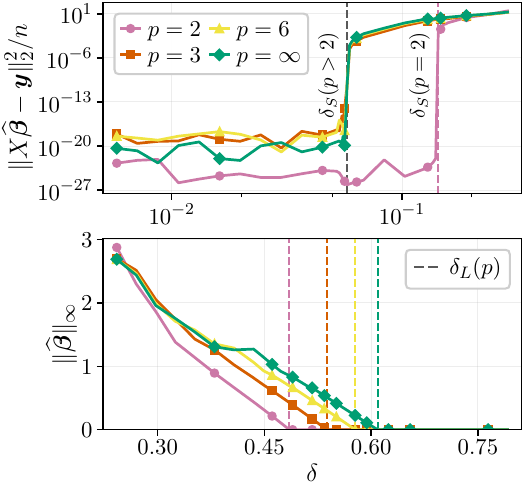}
  \caption{Large $\delta$ (bottom) and small $\delta$ (top) in the overparametrized regime, with zero solution and interpolation for $\delta \geq \delta_L$ and $\delta \leq \delta_S$, respectively.}
  \label{fig_paper_extreme_delta}
\end{wrapfigure}

\textbf{Fast rate.} For the fast rate, we consider a simple setup $y_i=\x_i^{\top} \param^* + \varepsilon_i$ with $d=10$, true parameter with $\beta^*_i = 3$ for the first 5 entries and otherwise zero (i.e., $s=5$), Gaussian noise $\vv{\varepsilon} \sim N(0,\sigma^2I)$ with $\sigma=0.5$, and entries $x_{ij} \in X$ independently and uniformly sampled from $[-1,1]$ (so $M=1$ in \Cref{thm_fast_rate_main_theorem_v2}). Here, the \emph{independent} sampling of $X$ ensures that $X \in \mathrm{RE}(s,3)$ in \Cref{thm_fast_rate_main_theorem_v2}; see the appendix for a standard derivation. We compute $\widehat{\param}\in\arg\min_\beta V_\delta(\param)$ and plot the in-sample error $\|X(\widehat{\param}-\param^*)\|_2^2/n$ as a function of $n$, shown in Figure \ref{fig_fast_and_slow_rate} (left). We clearly see the rate $O(n^{-1})$ as predicted by \Cref{thm_fast_rate_main_theorem_v2}. This rate hinges on the fact that $X \in \mathrm{RE}(s,3)$ due to the independent sampling. In contrast, if the entries in $X$ are \emph{correlated}, they may violate the RE condition and instead yield an $O(n^{-1/2})$ rate, as the next example illustrates.

\textbf{Slow rate.} For the slow rate, we consider $d=2$, $\param^* = [2, -1]^\top$, noise $\vv{\varepsilon} \sim N(0,\sigma^2I)$ with $\sigma=0.5$, and entries of $X$ still uniformly bounded but now \emph{correlated}. This correlation breaks the $\mathrm{RE}$ condition (see the appendix for details on $X$). Thus, we can only guarantee an $O(n^{-1/2})$ rate by \Cref{improved_in_sample_error_decay}, also seen in Figure \ref{fig_fast_and_slow_rate} (right).

\textbf{Small and large $\boldsymbol{\delta}$.} We conclude with the extreme regimes of $\delta$. We see that for large $\delta \geq \delta_L$ in Figure \ref{fig_paper_extreme_delta}, $\widehat{\param}$ becomes the zero solution as in \Cref{th_when_beta_zero}, while small $\delta \leq \delta_S$ in the overparametrized setting results in interpolation as in \Cref{th_when_beta_minimum_norm}. See the appendix for setup details.

\textbf{Saddle-point and $\boldsymbol{\eta}$-trick solver.} In the fast-rate experiment, both solvers achieve nearly identical prediction errors and robust-risk values; see the appendix for the full comparison. In Figure~\ref{fig_eta_saddle_comparison}, we see that the $\eta$-trick solver is faster at larger sample sizes, with the largest gains for $p=2$ and~$p=\infty$.

\section{Conclusion}\label{conclusion}

In this paper, we have studied Wasserstein DRO linear regression with $2 \leq p \leq \infty$. We first proved an equivalent robust quadratic form of the robust risk. Using this form, we generalized several important properties from the two special cases square-root Lasso ($p=2$) and adversarial linear regression ($p = \infty$) to the general $2 \leq p \leq \infty$ setting, including a slow error rate $O(n^{-1/2})$ in general, a fast rate $O(n^{-1})$ using a restricted eigenvalue condition, the pivotal property, and equivalent solutions when the ambiguity radius $\delta$ is small and large, respectively. We also proposed efficient numerical solvers and validated our theoretical findings via simulations.  

Here, we only study Wasserstein DRO linear regression. An interesting future direction is to investigate conditions under which the analysis in this paper applies to other loss functions or to nonparametric regression problems. For instance, some of the adversarial techniques we used have already been adapted to kernel methods~\citep{ribeiro_kernel_2025}.  Another future direction is to consider overparametrized regimes where one lets both $d$ and $n$ tend to infinity, under some fixed finite ratio $\rho := d/n$. The analytic properties in this paper do not explicitly depend on $d$ (except for \Cref{thm_fast_rate_main_theorem_v2}, where $s=d$ if no sparse solution exists), and may therefore naturally serve as a start. Finally, the deterministic error bounds in this paper depend only very mildly on $p$, with error rates independent of $p$. This is also reflected in the numerical simulations. It would be interesting if one could find cases studies in which different $p$ results in significantly different error (which could still be possible since we only provide upper bounds on the error).


\textbf{AI use statement.} In this work, we used generative AI tools to implement methods. We have also partially used generative AI tools to provide critical ingredients for proving mathematical claims. We have not used generative AI tools to help develop theoretical models or conceptual frameworks, formulate mathematical claims, assist in the writing of proofs and interpret results. We have also used AI tools to identify relevant literature (beyond the manual literature survey), search for information and improve readability. We have reviewed all AI-assisted work. In particular, the paper including all proofs are fully written and checked by the authors and the LLM-generated code is verified and tested for correctness.

\textbf{Reproducibility statement.} We provide the source code, model implementations and evaluation procedures in the anonymous repository linked (see the link on page 2).



\acksection
This work was supported by eSSENCE, ScilifeLab and the Wallenberg AI, Autonomous Systems and
Software Program (WASP), funded by the Knut and Alice Wallenberg Foundation. Computational resources were provided by the National Academic Infrastructure for Supercomputing in Sweden (NAISS), funded by the Swedish Research Council.

\bibliographystyle{iclr2027_conference}



\clearpage
\appendix
\renewcommand{\theHequation}{app.\theequation} 
\pagenumbering{roman}
\thispagestyle{empty}
\onecolumn
\etocdepthtag.toc{app}

{\hypersetup{allcolors=black} 

\begin{center}
    {\LARGE Distributionally robust linear regression \\ through the lens of adversarial training} \\[0.4cm]
    {\LARGE Appendix} \\[0.5cm]

    \begin{minipage}{0.9\textwidth}
        \itshape
    In this appendix, we provide additional material supporting the main text, in particular, all the proofs. The structure follows the order of the sections in the main text.
    \end{minipage}
\end{center}

\vspace{1cm}
\hrule
\vspace{0.3cm}

\etocsettagdepth{main}{-1}
\etocsettagdepth{app}{2}

\etocsettocstyle{\section*{Contents}}{}

\begin{center}
\begin{minipage}{0.9\textwidth}
    \etocarticlestyle
    \tableofcontents
\end{minipage}
\end{center}

\vspace{0.5cm}
\hrule
\vspace{1cm}

\newpage
}

\setcounter{equation}{0}
\renewcommand{\theequation}{A.\arabic{equation}}
\setcounter{figure}{0}
\renewcommand{\thefigure}{A.\arabic{figure}}


\section{Background: Additional details}

In this section, we provide additional details of Section \ref{background}.

\textbf{Proof of \Cref{remark_p_less_2_blows_up}.} We show that $V_\delta(\param) = \infty$ for $\param \neq 0$ if $1 \leq p <2$ as stated in \Cref{remark_p_less_2_blows_up}. This follows from \Cref{lemma_robust_risk_1D_formulation} by noting that the supremum inside the sum is infinite, since the quadratic term with $t_i$ ($\|\param\|_* \neq 0$) increases faster than $t_i^p$ (since $p<2$). Thus, $V_\delta(\param) = \infty$.

\section{Equivalent forms: Proofs and additional details}

In this section, we provide the proofs and additional details of Section \ref{equivalent_forms}.

\subsection{Proof of Theorem \ref{th_supremum_robust_risk}}

To prove \Cref{th_supremum_robust_risk}, we will use the following lemma that reduces the robust risk to several scalar problems:
\begin{lemma}\label{lemma_robust_risk_1D_formulation}
The robust risk of Wasserstein DRO linear regression for $1 \leq p<\infty$~equals
\begin{equation}\label{eq_lemma_robust_risk_1D_formulation}
  n V_\delta({\param})
  = \min_{\lambda \ge 0}
    \left[
      n\delta^p\lambda
      + \sum_{i=1}^{n} \sup_{t_i \ge 0}
        \Big\{ \big( \lvert r_i \rvert + t_i \lVert {\param} \rVert_* \big)^2 - \lambda t_i^p \Big\}
    \right].
\end{equation}
with $r_i := \x_i^\top\param-y_i$. 
\end{lemma}
\Cref{lemma_robust_risk_1D_formulation} is a direct application of Theorem 4 (ii) in \citet{shafiee2026nash}. For completeness, we provide a short proof here for the linear regression case:
\begin{proof}[Proof of \Cref{lemma_robust_risk_1D_formulation}]
We start from \Cref{optimal_transport_dro_dual}, where we note that the expectation in \eqref{eq:robust_risk_intermediate_form} turns into a sum due to the empirical distribution $\empProb := \frac{1}{n} \sum_{i=1}^n \delta_{\z_i}$, yielding:
\begin{equation}\label{eq_towards_scalar_problem_start}
  n V_\delta(\param) = 
  \min_{\lambda \ge 0}
    \left[
      n\delta^p \lambda
      + \sum_{i=1}^{n} \sup_{\z' \in \mathcal{Z}}
        \Big\{ h_{\param}(\z') - \lambda \|\z'-\z_i\| \Big\}
    \right].
\end{equation}
In our case, $\mathcal{Z} = \R^{d+1}$ and $h_{\param} : \R^{d+1} \to \R$ is the square loss combined with the linear predictor model, i.e.,
for $\z = (\x,y)$, we have $h_{\param}(\z) = \big({\param}^{\top} \x - y\big)^2$, and $\|\z'-\z\|$ with $z=(\x,y)$ and $z' = (\x',y')$ equals
\begin{equation*}
      \|\z'-\z\| =
  \begin{cases}
    \lVert \x' - \x \rVert ^p & \text{if } y = y', \\[2pt]
    +\infty & \text{otherwise.}
  \end{cases}
\end{equation*}
Since $\|\z'-\z\| = \infty$ whenever $y \neq y'$, the supremum in \eqref{eq_towards_scalar_problem_start} never moves $\z'$ away from $\z_i$ in the $y$-coordinate.\footnote{Indeed, such a deviation would let the outer minimization in \eqref{eq_towards_scalar_problem_start} pick any $\lambda>0$, resulting in an $-\infty$ cost due to $\|\z'-\z\| = \infty$, which is strictly suboptimal for the inner supremum (since the zero deviation $\z' = \z$ results in a bounded cost).} In other words, we have that
\begin{equation*}
  n V_\delta({\param})
  = \min_{\lambda \ge 0}
    \left[
      n\delta^p \lambda
      + \sum_{i=1}^{n} \sup_{\x' \in \R^d}
        \Big\{ \big({\param}^{\top} \x' - y_i\big)^2 - \lambda \lVert \x' - \x_i \rVert ^p \Big\}
    \right].
\end{equation*}
Now center the problem by letting $\boldsymbol{u}_i = \x' - \x_i$. Then
\[
  {\param}^{\top} \x' - y_i
  = {\param}^{\top}(\boldsymbol{u}_i + \x_i) - y_i
  = {\param}^{\top} \x_i - y_i + {\param}^{\top} \boldsymbol{u}_i
  = r_i + {\param}^{\top} \boldsymbol{u}_i ,
\]
and therefore
\[
  \big({\param}^{\top} \x' - y_i\big)^2 - \lambda \lVert \x' - \x_i \rVert ^p
  = \big(r_i + {\param}^{\top} \boldsymbol{u}_i \big)^2 - \lambda \lVert \boldsymbol{u}_i \rVert ^p .
\]
Thus, we have
\begin{equation*}
  n V_\delta({\param})
  = \min_{\lambda \ge 0}
    \left[
      n\delta^p \lambda
      + \sum_{i=1}^{n} \sup_{\boldsymbol{u} \in \R^d}
        \Big\{ \big(r_i + {\param}^{\top} \boldsymbol{u}\big)^2 - \lambda \lVert {\boldsymbol{u}} \rVert ^p \Big\}
    \right],
\end{equation*}
where we drop the subscript on $\boldsymbol{u}$ for convenience. We note that the penalty $\lambda \lVert {\boldsymbol{u}} \rVert ^p$ depends only on the magnitude of $\boldsymbol{u}$, and thus its
direction is irrelevant to the cost. To decompose direction and magnitude, let $\boldsymbol{u} = t \boldsymbol{v}$ with $t \ge 0$ and
$\lVert \boldsymbol{v} \rVert = 1$. Then, the supremum can be split into two parts:
\[
  \sup_{t \ge 0} \; \sup_{\lVert \boldsymbol{v} \rVert  = 1}
    \Big\{ \big(r_i + t{\param}^{\top} \boldsymbol{v}\big)^2 - \lambda t^p \Big\}
  = \sup_{t \ge 0}
    \left[ \sup_{\lVert \boldsymbol{v} \rVert  = 1} \Big\{ \big(r_i + t{\param}^{\top} \boldsymbol{v}\big)^2 \Big\} - \lambda t^p \right].
\]
The remaining inner problem $\sup_{\lVert \boldsymbol{v} \rVert  = 1} (r_i + t{\param}^{\top} \boldsymbol{v})^2$ can be
solved in closed form:
\begin{equation*}
    \sup_{\lVert \boldsymbol{v} \rVert  = 1} \Big\{ \big(r_i + t{\param}^{\top} \boldsymbol{v}\big)^2 \Big\} = 
    \sup_{\lVert \boldsymbol{v} \rVert  = 1} \Big\{ \big(|r_i| + t{\param}^{\top} \boldsymbol{v}\big)^2 \Big\} =
     \big(|r_i| + t \sup_{\lVert \boldsymbol{v} \rVert  = 1} \{ {\param}^{\top} \boldsymbol{v} \} \big)^2 =
     \big( \lvert r_i \rvert + t \lVert {\param} \rVert_* \big)^2,
\end{equation*}
where the first equality is due to symmetry, the second holds since $w \mapsto (a+tw)^2$ is increasing for $a,t \geq 0$, and the third holds since $\lVert {\param} \rVert_* = \sup_{\lVert \boldsymbol{v} \rVert  = 1} \{ {\param}^{\top} \boldsymbol{v} \}$. This leaves us with the scalar~problem
\begin{equation*}
  n V_\delta({\param})
  = \min_{\lambda \ge 0}
    \left[
      n\delta^p\lambda
      + \sum_{i=1}^{n} \sup_{t \ge 0}
        \Big\{ \big( \lvert r_i \rvert + t \lVert {\param} \rVert_* \big)^2 - \lambda t^p \Big\}
    \right],
\end{equation*}
completing the proof.
\end{proof}
We now prove \Cref{th_supremum_robust_risk}:
\begin{proof}[Proof of \Cref{th_supremum_robust_risk}]
We first consider the case $p < \infty$. We want to show that
\begin{equation}
F(\param) := \sup_{t_i \geq 0, \sum_{i=1}^n t_i^p \leq n \delta^p} \sum_{i=1}^n \big(|r_i|+t_i \|\param\|_*\big)^2. 
\label{eq:theorem_2_dual_cleaned}
\end{equation}
is equal to $nV_{\delta}(\param)$. For $\delta=0$, we trivially have $F(\param) = nV_{\delta}(\param) = \sum_{i=1}^n r_i^2$ using \eqref{eq_lemma_robust_risk_1D_formulation}. 
For the case $\delta>0$, we show equality using standard strong duality theory \cite[Chapter 5]{Boyd_Vandenberghe_2004_1}. More precisely, consider \eqref{eq:theorem_2_dual_cleaned}, and apply the change of variable $\nu_i = t_i^p$, leading to
\begin{align*}
    F(\param) = 
    \sup_{\nu_i \geq 0, \sum_{i=1}^n \nu_i \leq n \delta^p} \sum_{i=1}^n \big(|r_i|+\nu_i^{1/p} \|\param\|_*\big)^2 = \\
    \sup_{\nu_i \geq 0, \sum_{i=1}^n \nu_i \leq n \delta^p} \sum_{i=1}^n
    \underbrace{r_i^2 + 2 \lvert r_i \rvert \|\param\|_* \nu_i^{1/p} + \|\param\|_*^2 \nu_i^{2/p}}_{\phi_i(\nu_i) :=} =
    \sup_{\nu_i \geq 0, \sum_{i=1}^n \nu_i \leq n \delta^p} \sum_{i=1}^n \phi_i(\nu_i).
\end{align*}
Here, $\phi_i(\nu_i) = r_i^2 + 2 \lvert r_i \rvert \|\param\|_* \nu_i^{1/p} + \|\param\|_*^2 \nu_i^{2/p}$ is concave in $\nu_i \geq 0$ since $p \geq 2$ implies $1/p \in [0,1/2]$ and $2/p \in [0,1]$. In other words, $F(\param)$ is a concave maximization over an affine domain. Equivalently, to work with the more standard convex setting \cite[Chapter 5]{Boyd_Vandenberghe_2004_1}, we have that
\begin{equation}\label{eq_th2_dual_convex_form}
    \tilde{F}(\param) := -F(\param) = \inf_{\nu_i \geq 0, \sum_{i=1}^n \nu_i \leq n \delta^p} \sum_{i=1}^n -\phi_i(\nu_i)
\end{equation}
is a convex minimization over an affine domain. The dual problem of \eqref{eq_th2_dual_convex_form} with respect to the constraint $\sum_{i=1}^n \nu_i \leq n \delta^p$ (leaving $\nu_i \geq 0$ as part of the domain) is\footnote{We compactly write $\vv{\nu} \geq 0$ meaning $\nu_i \geq 0$ for each $i$.}
\begin{align}\label{eq_th2_dual2_convex_form}
    \sup_{\lambda \geq 0} \inf_{\vv{\nu} \geq 0} L(\nu,\lambda) := \sup_{\lambda \geq 0} \inf_{\vv{\nu} \geq 0} \left [ \sum_{i=1}^n (-\phi_i(\nu_i)) + \lambda \left ( \sum_{i=1}^n (\nu_i)-n\delta^{p} \right ) \right]
\end{align}
Moreover, $\nu_i = 0$ yields $\sum_{i=1}^n \nu_i =0 < n \delta^p$ for $\delta>0$, so Slater's condition holds. Therefore, strong duality holds \cite[Chapter 5]{Boyd_Vandenberghe_2004_1} and hence \eqref{eq_th2_dual_convex_form} is equal to \eqref{eq_th2_dual2_convex_form}. Equivalently,
\begin{align*}
    F(\param) = -\sup_{\lambda \geq 0} \inf_{\vv{\nu} \geq 0} \left [ \sum_{i=1}^n (-\phi_i(\nu_i)) + \lambda \left ( \sum_{i=1}^n (\nu_i)-n\delta^{p} \right ) \right] = \\
    \inf_{\lambda \geq 0} \sup_{\vv{\nu} \geq 0} \left [ \sum_{i=1}^n (\phi_i(\nu_i)) + \lambda \left ( n\delta^{p} - \sum_{i=1}^n \nu_i \right ) \right] =
    \inf_{\lambda \geq 0} \left [ n\delta^{p} \lambda +  \sum_{i=1}^n \sup_{\nu_i \geq 0} (\phi_i(\nu_i)-\lambda \nu_i) \right] = \\
    \inf_{\lambda \geq 0} \left [ n\delta^{p} \lambda +  \sum_{i=1}^n \sup_{t_i \geq 0} (\phi_i(t_i^p)-\lambda t_i^p) \right] =
    \inf_{\lambda \geq 0} \left [ n\delta^{p} \lambda +  \sum_{i=1}^n \sup_{t_i \geq 0} \left \{ \big( \lvert r_i \rvert + t_i \lVert {\param} \rVert_* \big)^2-\lambda t_i^p) \right \} \right] =\\  
    n V_\delta(\param)
\end{align*}
where the last equality is due to \eqref{eq_lemma_robust_risk_1D_formulation}. This completes the proof for $p<\infty$.

For $p=\infty$, we have by Theorem 2 in \citet{zhang_short_2024} that
\begin{equation*}
    n V_\delta(\param) = \sum_{i=1}^n \sup_{\|\vv{v}_i\| \leq \delta} ( (\x_i+\vv{v}_i)^T \param - y_i )^2
\end{equation*}
which can be rewritten to the desired form using a dual norm argument, see, e.g., Proposition 1 in~\citet{ribeiro_regularization_2023}. For completeness, we provide this argument here too:

\begin{align*}
\sup_{\|v_i\| \leq \delta} ( (\x_i+\vv{v}_i)^T \param - y_i )^2 =
\sup_{\|v_i\| \leq \delta} ( \vv{v}_i^T \param+ \x_i^T \param - y_i )^2 = 
\sup_{\|v_i\| \leq \delta} ( \vv{v}_i^T \param+ |\x_i^T \param - y_i |)^2 = \\
\sup_{\|v_i\| \leq 1} (\delta \vv{v}_i^T \param+ |\x_i^T \param - y_i |)^2 = 
(\delta \sup_{\|v_i\| \leq 1} \vv{v}_i^T \param+ |\x_i^T \param - y_i |)^2 =
(\delta \|\param\|_*+ |\x_i^T \param - y_i |)^2
\end{align*}
and thus
\begin{align*}
    V_{\delta}(\param) =
    \sum_{i=1}^n (|\x_i^T \param - y_i |+\delta \|\param\|_*)^2  = \sup_{t_i \geq 0, \|\vv{t}\| \leq \delta} \sum_{i=1}^n  (|\x_i^T \param - y_i |+t_i \|\param\|_*)^2
\end{align*}
completing the proof for $p = \infty$ and hence of \Cref{th_supremum_robust_risk}.
\end{proof}

\subsection{Proof of Corollary \ref{th_special_cases}}\label{proof_th_special_cases}
We continue by proving \Cref{th_special_cases}, using the saddle-point formulation of the robust risk.

\begin{proof}[Proof of \Cref{th_special_cases}] 
 We will use \Cref{th:saddle_point_form} (that we prove latter in the appendix). This result state that 
\[nV_{\delta}(\param) = \sup_{\vv{\alpha} \geq 0} K(\param,\vv{\alpha}) =  \sup_{\vv{\alpha} \geq 0} \left [ n^{1/2}\delta \|\param\|_* \|\vv{\alpha} \|_2+\sum_{i=1}^n \alpha_i |\x_i^{\top} \param - y_i | - \alpha_i^2/4 \right ]\]
 
We start with $p=2$ and compute \eqref{eq:saddle_point_form} noting that $q=2$. For brevity, let $D$ be the vector with entries $D_i =|\x_i^T \param - y_i |$. Then:
\begin{align*}
    nV_{\delta}(\param) 
    &=\sup_{\vv{\alpha} \geq 0} \left [ n^{1/2}\delta \|\param\|_* \|\vv{\alpha} \|_2+\langle \vv{\alpha}, \vv{D} \rangle - \|\alpha\|_2^2/4 \right ]  \\
    &= \sup_{\vv{\alpha} \geq 0} \left [ n^{1/2}\delta \|\param\|_* \|\vv{\alpha} \|_2+\|\vv{\alpha}\|_2 \|\vv{D}\|_2  - \|\vv{\alpha}\|_2^2/4 \right ] \\
     &= \sup_{t \geq 0} \left [ n^{1/2}\delta \|\param\|_* t+t \|\vv{D}\|_2  - t^2/4 \right ],
\end{align*}
where we used the Cauchy–Schwarz inequality $\langle \vv{\alpha}, \vv{D} \rangle \leq \|\vv{\alpha}\|_2 \|\vv{D}\|_2$ with equality if $\vv{\alpha} = l\cdot \vv{D}$ for some scalar $l \geq 0$ (note that such an $\vv{\alpha}$ is indeed feasible since $\vv{D}\geq0$), and then transformed the supremum to a one-dimensional optimization by setting $t=\|\vv{\alpha}\|_2$. This one-dimensional optimization is quadratic, with solution
\begin{equation*}
( \|\vv{D} \|_2+n^{1/2}\delta \|\param\|_*)^2.
\end{equation*}
Therefore
\begin{equation*}
    V_{\delta}(\param) = \left ( \frac{1}{n^{1/2}}  \|\vv{D} \|_2 +\delta \|\param\|_* \right)^2 = \left ( \sqrt{\frac{1}{n}\sum_{i=1}^n (\x_i^{\top} \param - y_i)^2}+\delta \|\param\|_* \right)^2,
\end{equation*}
which coincides with square-root Lasso \citep{belloni_squareroot_2011} for $\| \cdot \|_* = \| \cdot \|_1$, i.e., $\| \cdot \| = \| \cdot \|_\infty$. 

To show the adversarial linear regression case, we have by \Cref{th_supremum_robust_risk} that
\begin{align*}
    V_{\infty,\delta}(\param) = 
    \sup_{t_i \geq 0, \|t_i\| \leq \delta} \sum_{i=1}^n  (|\x_i^{\top} \param - y_i |+t_i \|\param\|_*)^2 = 
    \sum_{i=1}^n (|\x_i^{\top} \param - y_i |+\delta \|\param\|_*)^2 = \\
    \sum_{i=1}^n \sup_{\|\vv{v}_i\| \leq \delta} ( (\x_i+\vv{v}_i)^{\top} \param - y_i )^2
\end{align*}
where the last step follows by the dual norm argument in the proof of \Cref{th_supremum_robust_risk}. Here, we identify $v_i$ as the adversarial noise on the input $\x_i$. In other words, $p$-Wasserstein linear regression with $p = \infty$ coincides with adversarial linear regression. It remains to show that $\lim_{p \rightarrow \infty} V_{p,\delta}(\param) = V_{\infty,\delta}(\param)$. Again, using \Cref{th:saddle_point_form}, we have that:
\begin{equation*}
    \lim_{p \rightarrow \infty} V_{p,\delta}(\param) = \lim_{p \rightarrow \infty} \sup_{\vv{\alpha} \geq 0} K(\param,\vv{\alpha}) =  \lim_{p \rightarrow \infty} \sup_{\vv{\alpha} \geq 0} \left [ n^{1/p} \delta \|\param\|_* \|\vv{\alpha} \|_q+\sum_{i=1}^n \alpha_i |\x_i^{\top} \param - y_i | - \alpha_i^2/4 \right ].
\end{equation*}
We want to swap the limit and the supremum. To this end, fix $\param$ and let $f_p(\vv{\alpha}) := K(\param,\vv{\alpha})$, where we stress the dependence on $p$. We have that $\|\vv{\alpha}\|_q \leq n^{1/q} \| \vv{\alpha}  \|_2$ by a norm inequality and $\sum_{i=1}^n \alpha_i |\x_i^{\top} \param - y_i | \leq \|\vv{\alpha}\|_2 \|\vv{D} \|_2$ by Cauchy–Schwarz inequality. Thus,
\begin{equation*}
    f_p(\vv{\alpha}) \leq n^{1/p} \delta \|\param\|_* n^{1/q} \| \vv{\alpha}  \|_2 + \|\vv{\alpha}\|_2 \|\vv{D} \|_2 - \|\vv{\alpha}\|_2^2/4 = (n \delta \|\param\|_* + \|\vv{D} \|_2)\|\vv{\alpha}\|_2 - \|\vv{\alpha}\|_2^2/4
    \rightarrow -\infty
\end{equation*}
when $\|\vv{\alpha}\|_2 \rightarrow \infty$ independent of $p$. Thus, there exists a radius $R>0$ such that the maximum of $f_p(\vv{\alpha})$ for $\vv{\alpha} \geq 0$ is attained in $\Lambda := \{\vv{\alpha} \geq 0: \|\vv{\alpha}\|_2 \leq R\}$. Let 
\begin{equation*}
    f_\infty(\vv{\alpha}) := \lim_{p \rightarrow  \infty} f_p(\vv{\alpha}) = \delta \|\param\|_* \|\vv{\alpha} \|_1+\sum_{i=1}^n \alpha_i |\x_i^T \param - y_i | - \alpha_i^2/4
\end{equation*}
and note that $f_\infty(\vv{\alpha})$ also attains its maximum on $\Lambda$. Moreover, since $\Lambda$ is compact, $f_p(\vv{\alpha})$ converges uniformly to $f_\infty(\vv{\alpha})$ on $\Lambda$. Due to this uniform convergence, we can swap the limit and the~supremum:
\begin{equation*}
    \lim_{p \rightarrow \infty} \sup_{\vv{\alpha} \geq 0} f_p(\vv{\alpha}) = \lim_{p \rightarrow \infty} \sup_{\vv{\alpha} \in \Lambda} f_p(\vv{\alpha}) = \sup_{\vv{\alpha} \in \Lambda} \lim_{p \rightarrow \infty} f_p(\vv{\alpha}) = \sup_{\vv{\alpha} \in \Lambda} f_\infty(\vv{\alpha}) = \sup_{\vv{\alpha} \geq 0} f_\infty(\vv{\alpha})
\end{equation*}
Therefore,
\begin{align*}
    \lim_{p \rightarrow \infty} V_{p,\delta}(\param) = \lim_{p \rightarrow \infty} \sup_{\vv{\alpha} \geq 0} f_p(\vv{\alpha}) =  \sup_{\vv{\alpha} \geq 0} f_\infty(\vv{\alpha}) = \\
    \sup_{\vv{\alpha} \geq 0} \left [ \delta \|\param\|_* \|\vv{\alpha} \|_1+\sum_{i=1}^n \alpha_i |\x_i^{\top} \param - y_i | - \alpha_i^2/4 \right ] =
    \sup_{\vv{\alpha} \geq 0} \left [\sum_{i=1}^n  (\delta \|\param\|_*+ |\x_i^{\top} \param - y_i |)\alpha_i - \alpha_i^2/4 \right ] = \\
    \sum_{i=1}^n \sup_{\alpha_i \geq 0} \left [ (\delta \|\param\|_*+ |\x_i^{\top} \param - y_i |)\alpha_i - \alpha_i^2/4 \right ] = 
    \sum_{i=1}^n (\delta \|\param\|_*+ |\x_i^{\top} \param - y_i |)^2,
\end{align*}
The last expression coincides with $V_{\infty,\delta}(\param)$, completing the proof.
\end{proof}

\section{Error bounds}

\subsection{Slow rate: Proofs}

In this section, we provide the proofs of Section \ref{slow_rate}.

To prove \Cref{th_in_sample_error_prestage_improved}, we first show the following lemma:

\begin{lemma}\label{lemma_in_sample_error_prestage}
We have the following bounds:
\begin{align}
n V_{\delta}(\param^*) &\leq n \delta^2 \|\param^*\|_*^2+2n^{1/p} \delta \|\param^*\|_* \|\vv{\varepsilon}\|_q+\|\vv{\varepsilon}\|_2^2 \\
n V_{\delta}(\widehat{\param}) &\geq n \delta^2 \|\widehat{\param}\|_*^2+2\delta \|\widehat{\param}\|_* \|r(\widehat{\param})\|_1+\|r(\widehat{\param})\|_2^2,
\end{align}
where $r_i(\param) = |\x_i^T \param - y_i |$.
\end{lemma}

\begin{proof}
We obtain these bounds using the representation from \Cref{th_supremum_robust_risk}:
\begin{align*}
    n V_{\delta}(\param) = \max_{\vv{t} \geq 0, \|\vv{t}\|_p \leq 1} \sum_{i=1}^n \Big (n^{1/p}\delta \|\param\|_* t_i+ |\x_i^{\top} \param - y_i | \Big)^2 = \\
    \max_{\vv{t} \geq 0, \|\vv{t}\|_p \leq 1} \Big ( n^{2/p} \delta^2 \|\param\|_*^2 \|\vv{t}\|_2^2+2n^{1/p} \delta \|\param\|_* \sum_{i=1}^n t_i r_i(\param)+\|r(\param)\|_2^2 \Big ).
\end{align*}
To upper bound $V_{\delta}(\param^*)$, we note that $r(\param^*) = |\vv{\varepsilon}|$, resulting in
\begin{equation*}
    n V_{\delta}(\param^*) = \max_{\vv{t} \geq 0, \|\vv{t}\|_p \leq 1} \Big ( n^{2/p}\delta^2 \|\param^*\|_*^2 \|\vv{t}\|_2^2+2n^{1/p}\delta \|\param^*\|_* \sum_{i=1}^n t_i |\varepsilon_i|+\|\vv{\varepsilon}\|_2^2 \Big ).
\end{equation*}
Then we use Hölder's inequality $\sum_{i=1}^n t_i |\varepsilon_i| \leq \|\vv{t}\|_p \|\vv{\varepsilon}\|_q \leq \|\vv{\varepsilon}\|_q$ and the norm inequality $\|\vv{t}\|_2^2 \leq n^{1-2/p} \|\vv{t}\|_p^2 \leq n^{1-2/p}$ to get
\begin{equation*}
    n V_{\delta}(\param^*) \leq n \delta^2 \|\param^*\|_*^2+2n^{1/p} \delta \|\param^*\|_* \|\vv{\varepsilon}\|_q+\|\vv{\varepsilon}\|_2^2.
\end{equation*}
We next lower bound $n V_{\delta}(\widehat{\param})$. For this, we simply pick the point $s_i = n^{-1/p}$ (which satisfies $\|s\|_p =1$ and $s \geq 0$), yielding
\begin{equation*}
    n V_{\delta}(\widehat{\param}) \geq n \delta^2 \|\widehat{\param}\|_*^2+2\delta \|\widehat{\param}\|_* \|r(\widehat{\param})\|_1+\|r(\widehat{\param})\|_2^2.
\end{equation*}
This completes the proof.
\end{proof}

\textbf{Proof of \Cref{th_in_sample_error_prestage_improved}.} We now prove \Cref{th_in_sample_error_prestage_improved}:
\begin{proof}
By optimality, $n V_{\delta}(\widehat{\param}) \leq n V_{\delta}(\param^*)$. Combined with \Cref{lemma_in_sample_error_prestage} yields 
\begin{equation*}
    n \delta^2 \|\widehat{\param}\|_*^2+2\delta \|\widehat{\param}\|_* \|r(\widehat{\param})\|_1+\|r(\widehat{\param})\|_2^2 
    \leq 
    n \delta^2 \|\param^*\|_*^2+2n^{1/p}\delta \|\param^*\|_* \|\vv{\varepsilon}\|_q+\|\vv{\varepsilon}\|_2^2
\end{equation*}
Noting that $\|r(\widehat{\param})\|_2^2 = \|X\widehat{\Delta}\|_2^2-2\varepsilon^T X \widehat{\Delta}+\|\vv{\varepsilon}\|_2^2$, we get
\begin{equation*}
    n \delta^2 \|\widehat{\param}\|_*^2+2\delta \|\widehat{\param}\|_* \|r(\widehat{\param})\|_1+\|X\widehat{\Delta}\|_2^2 
    \leq 
    2\vv{\varepsilon}^{\top} X \widehat{\Delta}+
    n \delta^2 \|\param^*\|_*^2+2n^{1/p} \delta \|\param^*\|_* \|\vv{\varepsilon}\|_q.
\end{equation*}
The lemma now follows by dropping the nonnegative terms $n \delta^2 \|\widehat{\param}\|_*^2+2\delta \|\widehat{\param}\|_* \|r(\widehat{\param})\|_1$ and dividing by $n$.
\end{proof}

We continue with the proof of \Cref{th_controlled_dual_norm_improved_v3}, where we use the following two lemmas:

\begin{lemma}\label{lemma_in_sample_error_s_star}
Let $\vv{t}^*$ be a maximizer of \eqref{supremum_form_linear_regression} for $\param = \param^*$. Then
\begin{align}
\|X \widehat{\Delta} \|_2^2 \leq 2&\vv{\varepsilon}^{\top} X \widehat{\Delta}+2(\|\param^*\|_*-\|\widehat{\param}\|_*) \langle |\vv{\varepsilon}|,\vv{t}^*\rangle+ \notag \\ 
2& \|\widehat{\param} \|_* \langle |X\widehat{\Delta}|,\vv{t}^*\rangle+\|\vv{t}^*\|_2^2 (\|\param^*\|_*^2-\|\widehat{\param}\|_*^2).
\end{align}
\end{lemma}

\begin{proof}
Using \eqref{supremum_form_linear_regression} and the definition of $\vv{t}^*$, we get that
\begin{equation*}
n V_\delta(\param^*) =\sum_{i=1}^n (|x_i^T\param^*-y_i|+t_i^* \|\param^*\|_*)^2 = \sum_{i=1}^n (|\varepsilon_i|+t_i^* \|\param^*\|_*)^2
\end{equation*}
and
\begin{align*}
n V_\delta(\widehat{\param}) = \sup_{t_i \geq 0, \; \sum_{i=1}^n t_i^p \leq n \delta^p} \sum_{i=1}^n (|x_i^{\top}\widehat{\param}-y_i|+t_i \|\widehat{\param}\|_*)^2 \geq \\ 
    \sum_{i=1}^n (|x_i^{\top}\widehat{\param}-y_i|+t_i^* \|\widehat{\param}\|_*)^2 = \sum_{i=1}^n (|x_i^{\top} \widehat{\Delta}-\varepsilon_i|+t_i^* \|\widehat{\param}\|_*)^2.
\end{align*}
By optimality, $V_\delta(\widehat{\param}) \leq V_\delta(\param^*)$. Therefore,
\begin{equation*}
    \sum_{i=1}^n (|x_i^{\top} \widehat{\Delta}-\varepsilon_i|+t_i^* \|\widehat{\param}\|_*)^2 \leq n V_\delta(\widehat{\param}) \leq n V_\delta(\param^*) =  \sum_{i=1}^n (|\varepsilon_i|+t_i^* \|\param^*\|_*)^2.
\end{equation*}
Expanding both sides yields
\begin{equation*}
    \|X \widehat{\Delta}-\vv{\varepsilon}\|_2^2+2\|\widehat{\param}\|_*\langle |X\widehat{\Delta}-\vv{\varepsilon}|,\vv{t}^* \rangle+\|\vv{t}^*\|_2^2\|\widehat{\param}\|_*^2 \leq \|\vv{\varepsilon}\|_2^2+2\|\param^*\|_* \langle |\vv{\varepsilon}|,\vv{t}^*\rangle + \|\vv{t}^*\|_2^2 \|\param^*\|_*^2
\end{equation*}
which with $\|X \widehat{\Delta}-\vv{\varepsilon}\|_2^2 = \|X \widehat{\Delta}\|_2^2-2\vv{\varepsilon}^{\top} X \widehat{\Delta}+\|\vv{\varepsilon}\|_2^2$ results in
\begin{equation*}
    \|X \widehat{\Delta}\|_2^2+2\|\widehat{\param}\|_*\langle |X\widehat{\Delta}-\varepsilon|,\vv{t}^* \rangle+\|\vv{t}^*\|_2^2\|\widehat{\param}\|_*^2 \leq 2\vv{\varepsilon}^{\top} X \widehat{\Delta}+2\|\param^*\|_* \langle |\vv{\varepsilon}|,\vv{t}^*\rangle + \|\vv{t}^*\|_2^2 \|\param^*\|_*^2.
\end{equation*}
To get the final expression, we use the reverse triangle inequality
\begin{equation*}
    \langle |X\widehat{\Delta}-\vv{\varepsilon}|,\vv{t}^* \rangle \geq \langle |\vv{\varepsilon}|-|X\widehat{\Delta}|,\vv{t}^* \rangle = \langle |\vv{\varepsilon}|,\vv{t}^* \rangle - \langle |X\widehat{\Delta}|,\vv{t}^* \rangle
\end{equation*}
to obtain
\begin{equation*}
    \|X \widehat{\Delta}\|_2^2 \leq 2\vv{\varepsilon}^{\top} X \widehat{\Delta}+2(\|\param^*\|_*-\|\widehat{\param}\|_*) \langle |\vv{\varepsilon}|,\vv{t}^*\rangle + 2\|\widehat{\param}\|_* \langle |X\widehat{\Delta}|,\vv{t}^*\rangle+\|\vv{t}^*\|_2^2 (\|\param^*\|_*^2-\|\widehat{\param}\|_*^2).
\end{equation*}
This completes the proof.
\end{proof}

\begin{lemma}\label{lemma_to_bound_Delta}
We have
\begin{align*}
    0 \leq 2\|X^T \vv{\varepsilon}\| \| \widehat{\Delta} \|_*+2(\|\param^*\|_*-\|\widehat{\param}\|_*) \|\vv{\varepsilon}\|_q n^{1/p} \delta + n \delta^2 \|\param^*\|_*^2.
\end{align*}
\end{lemma}

\begin{proof}
By \Cref{lemma_in_sample_error_s_star}, we have that
\begin{align*}
        0 \leq \\ 2\varepsilon^T X \widehat{\Delta}+2(\|\param^*\|_*-\|\widehat{\param}\|_*) \langle |\vv{\varepsilon}|,\vv{t}^*\rangle + 2\|\widehat{\param}\|_* \langle |X\widehat{\Delta}|,\vv{t}^*\rangle-\|X \widehat{\Delta}\|_2^2+\|\vv{t}^*\|_2^2 (\|\param^*\|_*^2-\|\widehat{\param}\|_*^2) \leq \\
        2\vv{\varepsilon}^{\top} X \widehat{\Delta}+2(\|\param^*\|_*-\|\widehat{\param}\|_*) \langle |\varepsilon|,\vv{t}^*\rangle + 2\|\widehat{\param}\|_* \|X\widehat{\Delta}\|_2\|\vv{t}^*\|_2-\|X \widehat{\Delta}\|_2^2+\|\vv{t}^*\|_2^2 (\|\param^*\|_*^2-\|\widehat{\param}\|_*^2)
\end{align*}
where we used the Cauchy–Schwarz inequality for the last inequality. Note now that
\begin{equation*}
    2\|\widehat{\param}\|_* \|X\widehat{\Delta}\|_2\|\vv{t}^*\|_2-\|X \widehat{\Delta}\|_2^2 \leq \|\widehat{\param}\|_*^2 \|\vv{t}^*\|_2^2,
\end{equation*}
obtained by maximizing the expression with respect to $\|X\widehat{\Delta}\|_2$. Thus,
\begin{align*}
    0 \leq 2\vv{\varepsilon}^{\top} X \widehat{\Delta}+2(\|\param^*\|_*-\|\widehat{\param}\|_*) \langle |\vv{\varepsilon}|,\vv{t}^*\rangle +\|\vv{t}^*\|_2^2 \|\param^*\|_*^2.
\end{align*}
Finally, by the definition of the dual norm,
\begin{equation*}
2\vv{\varepsilon}^{\top} X \widehat{\Delta} \leq 2\|X^{\top} \vv{\varepsilon}\| \| \widehat{\Delta} \|_*,
\end{equation*}
Hölder's inequality
\begin{equation*}
    \langle |\vv{\varepsilon}|,\vv{t}^*\rangle \leq \|\vv{\varepsilon}\|_q \|\vv{t}^*\|_p \leq \|\vv{\varepsilon}\|_q n^{1/p} \delta,
\end{equation*}
and the norm inequality
\begin{equation*}
    \|\vv{t}^*\|_2^2 \leq n^{1-2/p} \|\vv{t}^*\|_p^2  \leq  n \delta^2,
\end{equation*}
we get
\begin{align*}
    0 \leq 2\|X^{\top} \vv{\varepsilon}\| \| \widehat{\Delta} \|_*+2(\|\param^*\|_*-\|\widehat{\param}\|_*) \|\vv{\varepsilon}\|_q n^{1/p} \delta + n \delta^2 \|\param^*\|_*^2.
\end{align*}
This completes the proof.
\end{proof}

\textbf{Proof of \Cref{th_controlled_dual_norm_improved_v3}.} We now prove \Cref{th_controlled_dual_norm_improved_v3}:

\begin{proof}
We have by \Cref{lemma_to_bound_Delta} and the triangle inequality $\| \widehat{\Delta} \|_* \leq \| \param^* \|_*+ \| \widehat{\param} \|_*$ that
\begin{align*}
0 \leq 2\|X^{\top} \vv{\varepsilon}\|(\| \param^* \|_*+ \| \widehat{\param} \|_*) +2(\| \param^* \|_*- \| \widehat{\param} \|_*) \|\vv{\varepsilon}\|_q n^{1/p} \delta + n \delta^2 \|\param^*\|_*^2.
\end{align*}
Moreover, $2\|X^T \vv{\varepsilon}\| < \delta n^{1/p} \|\vv{\varepsilon}\|_q$ by assumption. Hence,
\begin{align*}
0 \leq \delta n^{1/p} \|\vv{\varepsilon}\|_q(\| \param^* \|_*+ \| \widehat{\param} \|_*) +2(\| \param^* \|_*- \| \widehat{\param} \|_*) \|\vv{\varepsilon}\|_q n^{1/p} \delta + n \delta^2 \|\param^*\|_*^2  = \\
3 \delta n^{1/p} \|\vv{\varepsilon}\|_q \|\param^*\|_* - \delta n^{1/p} \|\vv{\varepsilon}\|_q \| \widehat{\param} \|_* + n \delta^2 \|\param^*\|_*^2
\end{align*}
which we can solve for $\| \widehat{\param} \|_*$ to get
\begin{equation*}
    \| \widehat{\param} \|_* \leq 3 \| \param^* \|_*+\frac{n^{1/q} \delta}{\|\vv{\varepsilon}\|_q} \| \param^* \|_*^2
\end{equation*}
from which the lemma follows. 
\end{proof}

\textbf{Proof of \Cref{improved_in_sample_error_decay}.} We continue with the proof of \Cref{improved_in_sample_error_decay}:

\begin{proof}
By \Cref{th_in_sample_error_prestage_improved}, we have that
\begin{equation*}
\frac{1}{n}\|X \widehat{\Delta}\|_2^2 \leq \frac{2}{n} \vv{\varepsilon}^{\top} X \widehat{\Delta}+ 2 \delta \frac{\|\vv{\varepsilon}\|_q}{n^{1/q}} \|\param^*\|_*+\delta^2 \|\param^*\|_*^2.
\end{equation*}
We bound the term $\frac{2}{n} \vv{\varepsilon}^{\top} X \widehat{\Delta}$ using the definition of a dual norm, the triangle inequality, the assumption $\delta >  2\frac{\|X^{\top} \vv{\varepsilon}\|}{n^{1/p}\|\vv{\varepsilon}\|_q}$, and \Cref{th_controlled_dual_norm_improved_v3}:
\begin{align*}
    \frac{2}{n} \vv{\varepsilon}^{\top} X \widehat{\Delta} \leq \frac{2}{n} \|X^{\top} \vv{\varepsilon}\| \|\widehat{\Delta}\|_* \leq \frac{2}{n} \|X^{\top} \vv{\varepsilon}\|(\|\widehat{\param}\|_*+\|\param^*\|_*) \leq \\
    \frac{2}{n} \|X^{\top} \vv{\varepsilon}\| \left (  4\|\param^*\|_*+\frac{n^{1/q}}{\|\vv{\varepsilon}\|_q}\|\param^*\|_*^2 \right ) \leq \\
    \frac{1}{n} \delta n^{1/p} \|\vv{\varepsilon}\|_q \left (  4\|\param^*\|_*+\frac{n^{1/q}}{\|\vv{\varepsilon}\|_q}\|\param^*\|_*^2 \right )
    = \\
    4 \delta \frac{\|\vv{\varepsilon}\|_q}{n^{1/q}} \|\param^*\|_*+\delta^2 \|\param^*\|_*^2,
\end{align*}
    from which we get
\begin{align*}
\frac{1}{n}\|X \widehat{\Delta}\|_2^2 \leq 6 \delta \frac{\|\vv{\varepsilon}\|_q}{n^{1/q}} \|\param^*\|_*+2\delta^2 \|\param^*\|_*^2.
\end{align*}
We have thus shown \eqref{improved_in_sample_error_decay_eq}. 

To show the high-probability rate in \eqref{eq_improved_in_sample_error_decay_rate}, note first that by a norm inequality (since $1\leq q \leq 2$):
\begin{equation*}
    6 \frac{\|\vv{\varepsilon}\|_q}{n^{1/q}} \leq 
    6 \frac{n^{1/q-1/2}\|\vv{\varepsilon}\|_2}{n^{1/q}} = 
    6 \frac{\|\vv{\varepsilon}\|_2}{\sqrt{n}}.
\end{equation*}
We bound $ f(\vv{\varepsilon}) = 6 \frac{\|\vv{\varepsilon}\|_2}{\sqrt{n}}$ with high probability. Since $f$ is Lipschitz with respect to the 2-norm with constant $L=\frac{6}{\sqrt{n}}$, a concentration bound for Lipschitz functions with Gaussian input \citep[Theorem 5.6]{10.1093/acprof:oso/9780199535255.001.0001} yields:
\begin{equation*}
    \probP(f(\vv{\varepsilon})-\E[f(\vv{\varepsilon})] \geq t) \leq \exp \left ( -\frac{t^2}{2\sigma^2L^2} \right ).
\end{equation*}
By setting
    $\gamma := \exp \left ( -\frac{t^2}{2\sigma^2L^2} \right )$
we can solve for $t$ to get 
    $t = \sigma L \sqrt{2 \log(1/\gamma)} = 6\sigma \sqrt{\frac{2 \log(1/\gamma)}{n}}$.
Thus with probability $\geq 1-\gamma$, we have that
    $f(\vv{\varepsilon}) < \E[f(\vv{\varepsilon})]+6\sigma \sqrt{\frac{2 \log(1/\gamma)}{n}},$
and since $\E[f(\vv{\varepsilon})] = 6\sigma$, we get that the bound 
\begin{equation*}
    6 \frac{\|\vv{\varepsilon}\|_q}{n^{1/q}} \leq 6\sigma + 6\sigma \sqrt{\frac{2 \log(1/\gamma)}{n}}
\end{equation*}
holds with probability greater than $1-\gamma$. Therefore, due to \eqref{improved_in_sample_error_decay_eq}, the rate of $\frac{1}{n}\|X \widehat{\Delta}\|_2^2$ is with high probability determined by the rate of $\delta$. That is, $\frac{1}{n}\|X \widehat{\Delta}\|_2^2 \in O(\delta)$ with high probability ($\geq 1- \gamma$). 

We now obtain the rate for the particular choice $\delta = K M \sqrt{ \frac{\log(d/\gamma)}{n}}$. To this end, we first upper bound $\bar{\delta}$. By a norm inequality
\begin{equation*}
    \bar{\delta} = \frac{2}{\sqrt{n}} \frac{\frac{\|X^{\top} \vv{\varepsilon}\|}{\sqrt{n}}}{\frac{\|\vv{\varepsilon}\|_q}{n^{1/q}}} \leq \frac{2}{\sqrt{n}} \frac{\frac{\|X^{\top} \vv{\varepsilon}\|}{\sqrt{n}}}{\frac{\|\vv{\varepsilon}\|_1}{n}}.
\end{equation*}
We upper bound $\|X^{\top} \vv{\varepsilon}\|$ and lower bound $\frac{\|\vv{\varepsilon}\|_1}{n}$. To lower bound $\frac{\|\vv{\varepsilon}\|_1}{n}$, note that $f(\vv{\varepsilon}) = \frac{\|\vv{\varepsilon}\|_1}{n} \leq \frac{\|\vv{\varepsilon}\|_2}{\sqrt{n}}$ is 2-norm Lipschitz with constant $L=1/\sqrt{n}$. Hence, a concentration bound for Gaussians \citep[Theorem 5.6]{10.1093/acprof:oso/9780199535255.001.0001} yield
\begin{equation*}
    \probP(\E[f(\vv{\varepsilon})]-f(\vv{\varepsilon}) \geq t) \leq \exp \left ( -\frac{t^2}{2\sigma^2L^2} \right ).
\end{equation*}
Setting $\gamma = \exp \left ( -\frac{t^2}{2\sigma^2L^2} \right )$ so that
    $t = \sigma L \sqrt{2 \log(1/\gamma)} = \sigma \sqrt{\frac{2 \log(1/\gamma)}{n}}$,
and since $\E[f(\vv{\varepsilon})] = \sigma \sqrt{\frac{2}{\pi}}$, we get with probability $\geq 1-\gamma$:
\begin{equation*}
    \frac{\|\vv{\varepsilon}\|_1}{n} \geq \sigma \sqrt{\frac{2}{\pi}}-\sigma \sqrt{\frac{2 \log(1/\gamma)}{n}}.
\end{equation*}
Moreover, we have with probability $\geq 1-2 \gamma$ that $\frac{\|X^{\top} \vv{\varepsilon}\|_\infty}{\sqrt{n}} \leq M \sigma \sqrt{2 \log(d/\gamma)}$ (see, e.g., \citep[Section C.3]{ribeiro_regularization_2023}), which implies that with probability $\geq 1-2 \gamma$:
\begin{equation*}
    \frac{\|X^{\top} \vv{\varepsilon}\|}{\sqrt{n}} \leq \frac{c\|X^{\top} \vv{\varepsilon}\|_\infty}{\sqrt{n}} \leq c M \sigma \sqrt{2 \log(d/\gamma)},
\end{equation*}
where we used that $\|\cdot \|_\infty \leq c\| \cdot \|$. Finally, we combine these two bounds to get a bound for $\bar{\delta}$. Concretely, with probability greater than $1-3\gamma$ (due to a union bound):
\begin{equation}\label{eq_upper_bound_bar_delta}
    \bar{\delta} = \frac{2}{\sqrt{n}} \frac{\frac{\|X^{\top} \vv{\varepsilon}\|}{\sqrt{n}}}{\frac{\|\vv{\varepsilon}\|_1}{n}} \leq \frac{2}{\sqrt{n}} \frac{cM \sigma \sqrt{2 \log(d/\gamma)}}{\sigma \sqrt{\frac{2}{\pi}}-\sigma \sqrt{\frac{2 \log(1/\gamma)}{n}}} = \frac{2}{\sqrt{n}} \frac{cM \sqrt{ \log(d/\gamma)}}{\sqrt{\frac{1}{\pi}}- \sqrt{\frac{ \log(1/\gamma)}{n}}} =: \frac{C_1(n)}{\sqrt{n}},
\end{equation}
valid for $n > \pi \log(1/\gamma)$ since the denominator is then positive. Note that the function $C_1(n)$ is decreasing with $n$ (given that $n > \pi \log(1/\gamma)$). In other words, $\bar{\delta} \in O(1/\sqrt{n})$ with probability $\geq 1-3\gamma$.\footnote{That we do not have an upper bound for $\bar{\delta}$ over the transition period $n \leq \pi \log(1/\gamma)$ is immaterial since $\bar{\delta}$ is almost surely finite over these finitely many value of $n$.} 

Finally, we tune $K$ in $\delta = K M \sqrt{ \frac{\log(d/\gamma)}{n}}$ to satisfy $\delta > \bar{\delta}$ for large enough $n$, which is true if
\begin{equation}\label{eq:tuning_K_slow_rate}
    K > \frac{2c}{\sqrt{\frac{1}{\pi}}- \sqrt{\frac{ \log(1/\gamma)}{n}}}
\end{equation}
for large enough $n$. Since $\frac{2c}{\sqrt{\frac{1}{\pi}}- \sqrt{\frac{ \log(1/\gamma)}{n}}} \rightarrow 2c\sqrt{\pi}$ as $n \rightarrow \infty$, a sufficient condition for satisfying $\delta > \bar{\delta}$ for large enough $n$ (with probability $\geq 1-3\gamma$) is that $K>2c\sqrt{\pi}$, where a larger $K$ results in $\delta > \bar{\delta}$ being satisfied sooner (smaller $n$). Thus, for $K>2c\sqrt{\pi}$, we have $\frac{1}{n}\|X \widehat{\Delta}\|_2^2 \in O(\delta) = O(1/\sqrt{n})$ (with probability $\geq 1- 4 \gamma$ due to a union bound), achieving the slow rate. This completes the proof. 
\end{proof}

\subsection{Fast rate: Proofs}

In this section, we provide the proofs of Section \ref{fast_rate}.

To prove \Cref{thm_fast_rate_main_theorem_v2}, we first need an intermediate result. This result holds for any norm $\| \cdot \|$ (not only $\| \cdot \| = \| \cdot \|_\infty$): 

\begin{lemma}\label{theorem_in_sample_controlled_by_Delta}
Let $\widehat{\param}$ be a minimizer of $V_\delta(\param)$ and consider $\widehat{\Delta} = \widehat{\param}-\param^*$. Assume each entry of $X$ is bounded by $M>0$ and let $r:= \sup_{\x\neq0} \frac{\|\x\|}{\|\x\|_{\infty}}$.\footnote{We introduce $r$ here to keep it general. Note that $r=1$ for $\|\cdot\| = \| \cdot \|_\infty$.} Then for  $\delta > \bar{\delta} = 2\frac{\|X^{\top} \vv{\varepsilon}\|}{n^{1/p} \|\vv{\varepsilon}\|_q}$, we have
\begin{align}
\frac{1}{n} \|X \widehat{\Delta}\|_2^2 \leq 
\delta \left ( 3 \frac{\|\vv{\varepsilon}\|_q}{n^{1/q}} + (2MCr+\delta C+\delta) \|\param^*\|_* \right ) \|\widehat{\Delta}\|_*
\end{align}
where $C = ( 3+\frac{n^{1/q} \delta}{\| \vv{\varepsilon}\|_q} \|\param^*\|_* )$.
\end{lemma}

\begin{proof}
By \Cref{lemma_in_sample_error_s_star}, we have that
\begin{equation*}
        \|X \widehat{\Delta}\|_2^2 \leq 2\vv{\varepsilon}^{\top} X \widehat{\Delta}+2(\|\param^*\|_*-\|\widehat{\param}\|_*) \langle |\vv{\varepsilon}|,\vv{t}^*\rangle + 2\|\widehat{\param}\|_* \langle |X\widehat{\Delta}|,\vv{t}^*\rangle+\|\vv{t}^*\|_2^2 (\|\param^*\|_*^2-\|\widehat{\param}\|_*^2).
\end{equation*}
To prove \Cref{theorem_in_sample_controlled_by_Delta}, we simply bound each term. Concretely, by definition of the dual norm and since $\delta > \bar{\delta} = 2\frac{\|X^{\top} \vv{\varepsilon}\|}{n^{1/p} \|\vv{\varepsilon}\|_q}$:
\begin{equation*}
2\vv{\varepsilon}^{\top} X \widehat{\Delta} \leq 2\|X^{\top} \vv{\varepsilon}\| \| \widehat{\Delta} \|_* \leq \delta n^{1/p} \|\vv{\varepsilon}\|_q \| \widehat{\Delta} \|_*.
\end{equation*}
Furthermore, the reverse triangle inequality yields
\begin{equation*}
    \|\param^*\|_*-\|\widehat{\param}\|_* \leq \|\param^*-\widehat{\param}\|_*  = \|\widehat{\Delta}\|_*,
\end{equation*}
and by Hölder's inequality
\begin{equation*}
\langle |\vv{\varepsilon}|,\vv{t}^*\rangle \leq \|\vv{\varepsilon}\|_q \|\vv{t}^*\|_p \leq \|\vv{\varepsilon}\|_q n^{1/p}\delta.
\end{equation*}
Therefore,
\begin{equation*}
    2(\|\param^*\|_*-\|\widehat{\param}\|_*) \langle |\vv{\varepsilon}|,\vv{t}^*\rangle \leq 2\|\vv{\varepsilon}\|_q n^{1/p}\delta \|\widehat{\Delta}\|_*
\end{equation*}

To bound $\langle |X\widehat{\Delta}|,\vv{t}^*\rangle$, note first that
\begin{align*}
    \langle |X\widehat{\Delta}|,\vv{t}^*\rangle = \sum_{i=1}^n \textrm{sign}(\x_i^{\top} \widehat{\Delta})\x_i^{\top} \widehat{\Delta}  t_i^*  \leq \left \| \sum_{i=1}^n \textrm{sign}(\x_i^{\top} \widehat{\Delta})\x_i  t_i^* \right \| \| \widehat{\Delta}\|_* \leq \\  \sum_{i=1}^n \|\x_i\|  t_i^*  \| \widehat{\Delta}\|_* \leq \| \vv{\tau}  \|_q \|\vv{t}^*\|_p \| \widehat{\Delta}\|_*
\end{align*}
where the last inequality holds due to Hölder's inequality with $\tau_i := \|\x_i\|$. Since $X$ has entries bounded by $M>0$, we get
\begin{equation*}
    \|\vv{\tau} \|_q \leq \left ( \sum_{i=1}^n (Mr)^q \right )^{1/q} = n^{1/q}M r
\end{equation*}
and by definition of $\vv{t}^*$,
\begin{equation*}
\|\vv{t}^*\|_p \leq n^{1/p}\delta,
\end{equation*}
and therefore
\begin{equation*}
    \langle |X\widehat{\Delta}|,\vv{t}^*\rangle \leq n^{1/q}M r n^{1/p}\delta \| \widehat{\Delta}\|_* = n M r \delta \| \widehat{\Delta}\|_*.
\end{equation*}
Thus, we have the bound
\begin{equation*}
    2\|\widehat{\param}\|_* \langle |X\widehat{\Delta}|,\vv{t}^*\rangle \leq 2\|\widehat{\param}\|_* n M r \delta \| \widehat{\Delta}\|_*.
\end{equation*}
Finally, we have that
\begin{equation*}
    \|\vv{t}^*\|_2^2 (\|\param^*\|_*^2-\|\widehat{\param}\|_*^2) = \|\vv{t}^*\|_2^2 (\|\param^*\|_*+\|\widehat{\param}\|_*) (\|\param^*\|_*-\|\widehat{\param}\|_*) \leq \|\vv{t}^*\|_2^2 (\|\param^*\|_*+\|\widehat{\param}\|_*) \|\widehat{\Delta}\|_*
\end{equation*}
by the reverse triangle inequality, and since
\begin{equation*}
\| \vv{t}^* \|_2 \leq n^{1/2-1/p} \|\vv{t}^*\|_p  \leq n^{1/2-1/p} n^{1/p}\delta = n^{1/2} \delta,
\end{equation*}
by a norm inequality, we get that $\|\vv{t}^* \|_2^2 \leq n \delta^2$ and so
\begin{equation*}
    \|\vv{t}^*\|_2^2 (\|\param^*\|_*^2-\|\widehat{\param}\|_*^2) \leq n \delta^2 (\|\param^*\|_*+\|\widehat{\param}\|_*) \|\widehat{\Delta}\|_*.
\end{equation*}
Combining all these bounds, we get
\begin{align*}
\|X \widehat{\Delta}\|_2^2 \leq 
\delta n^{1/p} \|\vv{\varepsilon}\|_q \| \widehat{\Delta} \|_* +
2\|\vv{\varepsilon}\|_q n^{1/p}\delta \|\widehat{\Delta}\|_* + \\
2\|\widehat{\param}\|_* n M r \delta \| \widehat{\Delta}\|_*  +
n \delta^2 (\|\param^*\|_*+\|\widehat{\param}\|_*) \|\widehat{\Delta}\|_* = \\
\Big ( 3\delta n^{1/p} \|\vv{\varepsilon}\|_q +
2\|\widehat{\param}\|_* n M r \delta  +
n \delta^2 (\|\param^*\|_*+\|\widehat{\param}\|_*) \Big ) \|\widehat{\Delta}\|_*
\end{align*}
and therefore
\begin{align*}
\frac{1}{n} \|X \widehat{\Delta}\|_2^2 \leq 
\delta \left ( 3 \frac{\|\vv{\varepsilon}\|_q}{n^{1/q}} +
2\|\widehat{\param}\|_* M r  +
\delta (\|\param^*\|_*+\|\widehat{\param}\|_*) \right ) \|\widehat{\Delta}\|_*.
\end{align*}
Finally, we bound $\|\widehat{\param}\|_*$ using \Cref{th_controlled_dual_norm_improved_v3} to obtain
\begin{align*}
\frac{1}{n} \|X \widehat{\Delta}\|_2^2 \leq 
\delta \left ( 3 \frac{\|\vv{\varepsilon}\|_q}{n^{1/q}} + (2MCr+\delta C+\delta) \|\param^*\|_* \right ) \|\widehat{\Delta}\|_*
\end{align*}
completing the proof.
\end{proof}

\textbf{Proof of \Cref{thm_fast_rate_main_theorem_v2}.} We now prove \Cref{thm_fast_rate_main_theorem_v2}:

\begin{proof}
We divide the proof into two cases, taking inspiration from the two cases in the proof of Theorem 2.3 in \citet{xie2024high}.

\textbf{Case 1: Assume $\|\widehat{\param}-\param^* \|_1+2(\|\param^*\|_1-\|\widehat{\param}\|_1) \leq 0$.} In this case, we will not need the RE condition. Indeed, note first that
\begin{equation*}
    \|\param^*\|_1-\|\widehat{\param}\|_1 = \|\widehat{\param}\|_1-\|\param^* \|_1+2(\|\param^*\|_1-\|\widehat{\param}\|_1) \leq \|\widehat{\param}-\param^* \|_1+2(\|\param^*\|_1-\|\widehat{\param}\|_1) \leq 0 
\end{equation*}

Next, using \Cref{lemma_in_sample_error_s_star} with $\| \cdot \|_* = \| \cdot \|_1$, we have:
\begin{equation*}
        \|X \widehat{\Delta}\|_2^2 \leq 2\vv{\varepsilon}^{\top} X \widehat{\Delta}+2(\|\param^*\|_1-\|\widehat{\param}\|_1) \langle |\vv{\varepsilon}|,\vv{t}^*\rangle + 2\|\widehat{\param}\|_1 \langle |X\widehat{\Delta}|,\vv{t}^*\rangle+\|\vv{t}^*\|_2^2 (\|\param^*\|_1^2-\|\widehat{\param}\|_1^2).
\end{equation*}
We bound the first term as
\begin{equation*}
    2\vv{\varepsilon}^{\top} X \widehat{\Delta} \leq 2 \|X^{\top} \vv{\varepsilon}\|_\infty \|\widehat{\Delta} \|_1 \leq \delta n^{1/p} \|\vv{\varepsilon}\|_q \|\widehat{\Delta} \|_1
\end{equation*}
using a dual norm inequality and the assumption $\delta > \bar{\delta} = 2\frac{\|X^{\top} \vv{\varepsilon}\|}{n^{1/p} \|\vv{\varepsilon}\|_q}$. Moreover, we bound the second term as
\begin{equation*}
    \langle |\vv{\varepsilon}|,\vv{t}^*\rangle \leq \|\vv{\varepsilon}\|_q \|\vv{t}^*\|_p \leq \|\vv{\varepsilon}\|_q n^{1/p} \delta,
\end{equation*}
using Hölder's inequality and the $p$-norm constraint on $\vv{t}^*$. Combining these bounds, we get
\begin{align*}
        \|X \widehat{\Delta}\|_2^2 \leq \\
        \delta n^{1/p} \|\vv{\varepsilon}\|_q \|\widehat{\Delta} \|_1+2(\|\param^*\|_1-\|\widehat{\param}\|_1) \|\vv{\varepsilon}\|_q n^{1/p} \delta + 2\|\widehat{\param}\|_1 \langle |X\widehat{\Delta}|,\vv{t}^*\rangle+\|\vv{t}^*\|_2^2 (\|\param^*\|_1^2-\|\widehat{\param}\|_1^2) = \\
        \delta n^{1/p} \|\vv{\varepsilon}\|_q \big [ \|\widehat{\Delta} \|_1+2(\|\param^*\|_1-\|\widehat{\param}\|_1) \big ] + 2\|\widehat{\param}\|_1 \langle |X\widehat{\Delta}|,\vv{t}^*\rangle+\|\vv{t}^*\|_2^2 (\|\param^*\|_1^2-\|\widehat{\param}\|_1^2) \leq \\
        2\|\widehat{\param}\|_1 \langle |X\widehat{\Delta}|,\vv{t}^*\rangle.
\end{align*}
where we used the assumption $\|\widehat{\param}-\param^* \|_1+2(\|\param^*\|_1-\|\widehat{\param}\|_1) \leq 0$ and that $\|\param^*\|_1-\|\widehat{\param}\|_1 \leq 0$ in the last step. Moreover, by a norm inequality,
\begin{equation*}
\| \vv{t}^* \|_2 \leq n^{1/2-1/p} \|\vv{t}^*\|_p  \leq n^{1/2-1/p} n^{1/p}\delta = n^{1/2} \delta,
\end{equation*}
we get
\begin{equation*}
    \langle |X\widehat{\Delta}|,\vv{t}^*\rangle \leq \|X\widehat{\Delta}\|_2 \|\vv{t}^*\|_2 \leq \|X\widehat{\Delta}\|_2 n^{1/2} \delta,
\end{equation*}
using the Cauchy–Schwarz inequality. Hence, $\|X\widehat{\Delta}\|_2^2 \leq 2\|\widehat{\param}\|_1 \|X\widehat{\Delta}\|_2 n^{1/2} \delta$ which implies $\frac{1}{\sqrt{n}} \|X \widehat{\Delta}\|_2 \leq 2\delta\|\widehat{\param}\|_1$ and so
\begin{align*}
        \frac{1}{n} \|X \widehat{\Delta}\|_2^2 \leq 4\delta^2\|\widehat{\param}\|_1^2.
\end{align*}
We finally bound $\|\widehat{\param}\|_1^2$ using \Cref{th_controlled_dual_norm_improved_v3} yielding $\|\widehat{\param}\|_1^2 \leq C^2 \|\param^*\|_1^2$ and
\begin{align*}
        \frac{1}{n} \|X \widehat{\Delta}\|_2^2 \leq 4\delta^2C^2 \|\param^*\|_1^2.
\end{align*}

\textbf{Case 2: Assume $\|\widehat{\param}-\param^* \|_1+2(\|\param^*\|_1-\|\widehat{\param}\|_1) \geq 0$.} In this case, we will use the RE condition. More precisely, we have
\begin{equation*}
    \|\widehat{\Delta} \|_1 \leq \|\widehat{\Delta}_S \|_1 + \|\widehat{\Delta}_{S^c} \|_1
\end{equation*}
and
\begin{align*}
    \|\param^*\|_1-\|\widehat{\param}\|_1 = \|\param^*_S\|_1 - \|\param^*_S+\widehat{\Delta}_S+\widehat{\Delta}_{S^c} \|_1 = \\ \|\param^*_S\|_1 - \|\param^*_S+\widehat{\Delta}_S \|_1 - \| \widehat{\Delta}_{S^c} \|_1 \leq \|\widehat{\Delta}_S\|_1 - \| \widehat{\Delta}_{S^c} \|_1.
\end{align*}
Hence,
\begin{equation*}
    0 \leq \|\widehat{\Delta} \|_1+2(\|\param^*\|_1-\|\widehat{\param}\|_1) \leq \|\widehat{\Delta}_S \|_1 + \|\widehat{\Delta}_{S^c} \|_1 + 2(\|\widehat{\Delta}_S\|_1 - \| \widehat{\Delta}_{S^c} \|_1)
\end{equation*}
which implies that $\|\widehat{\Delta}_{S^c} \|_1 \leq 3 \|\widehat{\Delta}_S\|_1$. In other words, we can use $\mathrm{RE}(s,3)$ to get
\begin{equation*}
\|\widehat{\Delta}\|_1 = \|\widehat{\Delta}_S\|_1+\|\widehat{\Delta}_{S^c}\|_1 \leq 4 \|\widehat{\Delta}_S\|_1 \leq 4s \|\widehat{\Delta}_S\|_2 \leq 4s \frac{\|X\widehat{\Delta}\|_2}{\kappa \sqrt{n}},
\end{equation*}
where we also used the norm inequality $\|\widehat{\Delta}_S\|_1 \leq s \|\widehat{\Delta}_S\|_2$. This bound combined with \Cref{theorem_in_sample_controlled_by_Delta}~yields
\begin{align*}
    \frac{1}{n} \|X\widehat{\Delta}\|_2^2 \leq \delta B \|\widehat{\Delta}\|_1 \leq 4\delta B s \frac{\|X\widehat{\Delta}\|_2}{\kappa \sqrt{n}} \quad \Rightarrow \quad \frac{1}{\sqrt{n}} \|X\widehat{\Delta}\|_2 \leq \frac{4\delta B s}{\kappa}.
\end{align*}
Therefore, $\frac{1}{n} \|X\widehat{\Delta}\|_2^2 \leq \frac{16\delta^2 B^2 s^2}{\kappa^2}$, completing the second case.

Combining Cases 1 and 2, we conclude that
\begin{align*}
\frac{1}{n} \|X\widehat{\Delta}\|_2^2 \leq \delta^2 \cdot \max \left \{ 4C^2 \|\param^*\|_1^2, \frac{16 B^2 s^2}{\kappa^2} \right \}.
\end{align*}
That is, we have proved \eqref{eq_fast_rate_main_theorem_v2_bound}. It remains to show the rates. To this end, assume $\vv{\varepsilon}  \sim N(0, \sigma  I)$. We show that $B$ and $C$ are bounded by constants with high probability. We start with $C$. By a norm inequality and the proof of \Cref{improved_in_sample_error_decay}, we have with probability $\geq 1-\gamma$:
\begin{equation*}
    \frac{\|\vv{\varepsilon}\|_q}{n^{1/q}} \geq \frac{\|\vv{\varepsilon}\|_1}{n} \geq \sigma \sqrt{\frac{2}{\pi}}-\sigma \sqrt{\frac{2 \log(1/\gamma)}{n}}.
\end{equation*}
Therefore, with probability $\geq 1-\gamma$:
\begin{equation*}
    C =  \left (  3+\frac{n^{1/q} \delta}{\| \vv{\varepsilon}\|_q} \|\param^*\|_* \right ) \leq \left (  3+\frac{\delta}{\sigma \sqrt{\frac{2}{\pi}}-\sigma \sqrt{\frac{2 \log(1/\gamma)}{n}}} \|\param^*\|_* \right ),
\end{equation*}
valid for $n > \pi \log(1/\gamma)$ (to ensure the denominator is positive). In other words, $C$ is bounded by a constant with high probability.\footnote{For the finitely many $n$ such that $n \leq \pi \log(1/\gamma)$, we have that $C$ is almost surely bounded.} Moreover, by the proof of \Cref{improved_in_sample_error_decay}, we have with probability $\geq 1-\gamma$:
\begin{equation*}
    \frac{\|\vv{\varepsilon}\|_q}{n^{1/q}} \leq \sigma + \sigma \sqrt{\frac{2 \log(1/\gamma)}{n}}.
\end{equation*}
Hence, applying a union bound, we have with probability $\geq 1-2\gamma$:
\begin{align*}
    B = 3 \frac{\|\vv{\varepsilon}\|_q}{n^{1/q}} + (2M+\delta )\|\param^*\|_*C +\delta\|\param^*\|_* \leq \\ 
    3 \left ( \sigma + \sigma \sqrt{\frac{2 \log(1/\gamma)}{n}} \right ) + (2M+\delta )\|\param^*\|_* \left (  3+\frac{\delta}{\sigma \sqrt{\frac{2}{\pi}}-\sigma \sqrt{\frac{2 \log(1/\gamma)}{n}}} \|\param^*\|_* \right ) +\delta\|\param^*\|_*  
\end{align*}
valid for $n > \pi \log(1/\gamma)$. We conclude that $B$ is also bounded by a constant with high probability. As a result, we have that $\frac{1}{n} \|X\widehat{\Delta}\|_2^2 \in O(\delta^2)$ with high probability ($\geq 1-2\gamma$), independent of $p$ since the upper bounds on $B$ and $C$ are independent of $p$. Finally, by setting $\delta$ as in \Cref{improved_in_sample_error_decay} and following its proof (where $\delta>\bar{\delta}$ with probability $\geq 1-3\gamma$), we analogously obtain that $\frac{1}{n} \|X\widehat{\Delta}\|_2^2 \in O(\delta^2) = O(1/n)$ for $K>2c\sqrt{\pi}$ with probability $\geq 1-5\gamma$ due to a union bound. This completes the proof. 
\end{proof}

\section{Small and large ambiguity balls: Proofs and additional details}

In this section, we provide the proofs and additional details of Section \ref{additional_properties}.

\subsection{Proof of Theorem \ref{th_when_beta_zero}} 
To prove \Cref{th_when_beta_zero}, we will use a generalization of Danskin's theorem \citep{danskin1967theory} proved in the PhD thesis of Dimitri P. Bertsekas \citep[Proposition A.22]{bertsekas1971control}. Indeed, \Cref{th_general_danskin} is a straightforward special case of \citep[Proposition A.22]{bertsekas1971control} adapted to our setting:\footnote{The subdifferential of a function $\varphi: \R^d \rightarrow \R$ at a point $\z \in \R^d$ is the set 
$$\{ \vv{v} \in \R^d: \varphi(\z')-\varphi(\z) \geq \langle \vv{v},\z'-\z  \rangle,\forall \z' \in \R^d \}.$$
For standard properties of subdifferentials, see, e.g., \citep{bertsekas_convex_2003,clarke_optimization_1990,boyd_subgradients_2022}.
}

\begin{theorem}[{\citealp[Proposition A.22]{bertsekas1971control}, Simplified}]\label{th_general_danskin}
Let $\phi: \R^d \times \R^n \rightarrow \R$ be a continuous function such that $\phi(\cdot,\vv{t})$ is convex for each $\vv{t} \in T$ in a compact set $T$, and consider $f(\param):= \sup_{\vv{t} \in T} \phi(\param,\vv{t})$. Then, for each $\param \in R^d$, the subdifferential of $f$ equals
\begin{equation}
\partial f(\param) = \mathrm{conv}\{ \partial_{\param} \phi(\param,\vv{t}): \vv{t} \in T(\param) \}, \quad T(\param) := \{\vv{t} \in T: \phi(\param,\vv{t}) = \max_{\vv{t} \in T} \phi(\param,\vv{t}) \},
\end{equation}
where $\partial_{\param} \phi(\param,\vv{t})$ is the subdifferential of $\phi(\cdot,\vv{t})$ at $\param$.
\end{theorem}

\begin{proof}[Proof of \Cref{th_when_beta_zero}]
Note that $\widehat{\param} = \vv{0}$ minimizes the robust risk $V_\delta(\param)$ if and only if $\vv{0} \in  \partial V_\delta(\vv{0})$, since $V_\delta(\param)$ is convex in $\param$. We compute $\partial V_\delta(\vv{0})$ using \Cref{th_general_danskin}. More precisely, note that $nV_\delta(\param) = \sup_{\vv{t} \in T} \phi(\param,\vv{t})$ with
\begin{align*}
\phi(\param,\vv{t}) &= \sum_{i=1}^n \left ( |\x_i^{\top} \param-y_i|+t_i \|\param \|_* \right )^2 \\
T &=  \left \{ \vv{t} \in \R^n:  t_i \geq 0, \|\vv{t}\|_p \leq n^{1/p} \delta \right \}
\end{align*}
It is easy to verify that the conditions of \Cref{th_general_danskin} hold, yielding:
\begin{equation}\label{eq_subgrad_robust_risk_at_zero}
    \partial_{\param} (nV_{\delta}(\vv{0})) = \mathrm{conv}\{ \partial_{\param} \phi(\vv{0},\vv{t}): \vv{t} \in T(\vv{0}) \}.
\end{equation}
We compute $\partial_{\param} \phi(\vv{0},\vv{t})$ and $T(\vv{0})$. In general,
\begin{align}\label{eq_subgrad_f_s_general}
\partial_{\param} \phi(\param,\vv{t}) = 2\sum_{i=1}^n \left ( |\x_i^{\top} \param-y_i|+t_i \|\param \|_* \right ) \left ( \partial_{\param} |\x_i^{\top} \param-y_i|+t_i \partial_{\param} \|\param \|_* \right ), 
\end{align}
where
\begin{align*}
\partial_{\param} |\x_i^{\top} \param-y_i| &= \x_i \mathrm{sign}( \x_i^{\top} \param-y_i) \quad \text{with} \quad 
\mathrm{sign}(a) :=
\begin{cases}
    \{1\} &\text{ if } a>0  \\
    \{-1\} &\text{ if } a<0 \\
    [-1,1] &\text{ if } a=0 
\end{cases}
\\
\partial_{\param} \| \param \|_* &= \{ \z \in \R^d: \|\z\| \leq 1, \langle \z,\param \rangle = \|\param \|_* \}
\end{align*}
by standard derivations.\footnote{Note that we are working with sets where addition is the Minkowski addition, standard for~subdifferentials.} In particular, for $\param = \vv{0}$, we have $\partial_{\param} \| \param \|_* = \mathbb{B} := \{\z \in \R^d: \|\z\|\leq 1\}$ and $\partial_{\param} |\x_i^{\top} \param-y_i|  = -\x_i \mathrm{sign}(y_i)$. Hence,
\begin{align*}
\partial_{\param} \phi(\vv{0},\vv{t}) = 
2\sum_{i=1}^n |y_i| \left ( -\x_i \mathrm{sign}(y_i) +t_i \mathbb{B} \right ) = \\ 
2\sum_{i=1}^n \left ( -\x_i y_i \right ) + 2 \sum_{i=1}^n \left ( t_i |y_i| \mathbb{B} \right ) = 2\big ( -X^T \y+\langle \vv{t},|\y| \rangle \mathbb{B} \big )
\end{align*}
where we used that $|y_i|\mathrm{sign}(y_i)=y_i$, and that the Minkowski sum of balls with the same center (the origin) equals the ball with their radii added. Next, we note that $\phi(\vv{0},\vv{t}) = \sum_{i=1}^n y_i^2$ for each $\vv{t} \in T$, so $T(\vv{0}) = T$. Therefore, \eqref{eq_subgrad_robust_risk_at_zero} yields
\begin{align*}
    \partial_{\param} (nV_{\delta}(\vv{0}))  = 
    \mathrm{conv} \left [ \bigcup_{\vv{t} \in T} 2\big ( -X^{\top} \y+\langle \vv{t},|\y| \rangle \mathbb{B} \big ) \right ]  = \\
     -X^{\top} \y+\big (\max_{\vv{t} \in T} \langle \vv{t},|\y| \rangle \big ) \mathbb{B} = -X^{\top} \y + n^{1/p} \delta \|\y\|_q \mathbb{B}
\end{align*}
since the convex hull of a union of balls with the same center equals the ball with largest radius (i.e., $\max_{\vv{t} \in T} \langle \vv{t},|\y| \rangle$ in our case) and
\begin{align*}
\max_{\vv{t} \in T} \langle \vv{t},|\y| \rangle = n^{1/p} \delta \max_{\|\vv{t}\|_p \leq 1} \langle \vv{t},|\y| \rangle = n^{1/p} \delta \|\y\|_q.
\end{align*}
Therefore, $\vv{0} \in \partial_{\param} V_{\delta}(\vv{0})$ if and only if
\begin{align*}
\vv{0} \in -X^{\top} \y + n^{1/p} \delta \|\y\|_q \mathbb{B} \quad \Leftrightarrow \quad \|X^{\top} \y \| \leq n^{1/p} \delta \|\y\|_q \quad \Leftrightarrow \quad \frac{\|X^{\top} \y \|}{n^{1/p}\|\y\|_q} \leq  \delta
\end{align*}
completing the proof.
\end{proof}

\subsection{Proof of Theorem \ref{th_when_beta_minimum_norm}} 
We continue with the proof of \Cref{th_when_beta_minimum_norm}. We need the following lemma:

\begin{lemma}\label{lemma_when_beta_minimum_norm}
Let $\bar{\param}$ be the minimum $\|\cdot\|_*$-norm interpolator. Then $X^T \widehat{\vv{\alpha}} \in \partial \|\bar{\param} \|_*$ if and only if $\widehat{\vv{\alpha}}$ solves $\max_{\|X^{\top} \vv{\alpha}\| \leq 1} \vv{\alpha}^{\top} \y$.
\end{lemma}
\begin{proof}
Analogously to the proof of Lemma 1 in \citet{ribeiro_regularization_2023}, we have by strong duality that
\begin{align*}
    \| \bar{\param} \|_* = \min_{X \param = \y} \|\param  \|_* = \max_{\vv{\alpha}} \min_{\param} \left ( \|\param  \|_*+ \vv{\alpha}^{\top}(\y-X \param)  \right ) = \\
    \max_{\vv{\alpha}} \left ( \vv{\alpha}^{\top} \y + \min_{\param} (\|\param  \|_*- \vv{\alpha}^{\top} X \param) \right) = \max_{\|X^{\top}\vv{\alpha}\|\leq 1} \vv{\alpha}^{\top} \y,
\end{align*}
where the last step used that the Fenchel conjugate of $\|\param  \|_*$ is the indicator function on the unit ball with norm~$\|\cdot \|$. Moreover, 
\begin{align*}
\partial \| \bar{\param} \|_* &= \{ \z \in \R^d: \|\z\| \leq 1, \z^{\top} \bar{\param} = \| \bar{\param} \|_* \}
\end{align*}
so $X^{\top} \widehat{\vv{\alpha}} \in \partial \| \bar{\param} \|_*$ implies that $\| X^{\top} \widehat{\vv{\alpha}} \| \leq 1$ and 
\begin{align*}
\widehat{\vv{\alpha}}^{\top} \y = \widehat{\vv{\alpha}}^{\top} X\bar{\param} = (X^{\top} \widehat{\vv{\alpha}})^{\top} \bar{\param}  =   \| \bar{\param} \|_* = \max_{\|X^{\top} \vv{\alpha}\|\leq 1} \vv{\alpha}^{\top} \y.
\end{align*}
That is, $\widehat{\vv{\alpha}}$ solves $\max_{\|X^{\top} \vv{\alpha}\| \leq 1} \vv{\alpha}^{\top} \y$. Conversely, assume $\widehat{\vv{\alpha}}$ solves $\max_{\|X^{\top} \vv{\alpha}\| \leq 1} \vv{\alpha}^{\top} \y$. Then $\| X^{\top} \widehat{\vv{\alpha}} \| \leq 1$ and
\begin{equation*}
    (X^{\top} \widehat{\vv{\alpha}})^{\top} \bar{\param} = \widehat{\vv{\alpha}}^{\top} \y = \max_{\|X^{\top} \vv{\alpha}\|\leq 1} \vv{\alpha}^{\top} \y = \| \bar{\param} \|_*
\end{equation*}
so $X^{\top} \widehat{\vv{\alpha}} \in \partial \| \bar{\param} \|_*$, completing the proof.
\end{proof}

\begin{proof}[Proof of \Cref{th_when_beta_minimum_norm}]
Note that $\bar{\param}$ minimizes $V_{\delta}(\param)$ if and only if $\vv{0} \in \partial_{\param} V_{\delta}(\bar{\param})$. We divide the proof into two subcases: $2 <p \leq \infty$ and $p=2$.

Assume first that $2 <p \leq \infty$. We compute $\partial_{\param} V_{\delta}(\bar{\param})$ using \Cref{th_general_danskin}. To this end, let $\phi$ and $T$ be as in the proof of \Cref{th_when_beta_zero}, so that $nV_\delta(\param) = \sup_{\vv{t} \in T} \phi(\param,\vv{t})$. \Cref{th_general_danskin} implies that
\begin{equation}\label{eq_subgrad_robust_risk_small_delta}
    \partial (nV_\delta(\bar{\param})) = \mathrm{conv}\{ \partial_{\param} \phi(\bar{\param},\vv{t}): \vv{t} \in T(\bar{\param}) \}.
\end{equation}
We compute $\partial_{\param} \phi(\bar{\param},\vv{t})$ and $T(\bar{\param})$. By \eqref{eq_subgrad_f_s_general}, we have
\begin{align*}
\partial_{\param} \phi(\bar{\param},\vv{t}) = 2 \sum_{i=1}^n t_i \|\bar{\param}\|_* (\x_i \mathbb{I}+t_i \partial_{\param} \|\bar{\param}\|_*)
\end{align*}
since $X \bar{\param} = \y$, where $\mathbb{I} = [-1,1]$. Moreover,
\begin{align*}
\phi(\bar{\param},\vv{t}) = \sum_{i=1}^n t_i^2 \|\bar{\param}\|_*^2 = \|\bar{\param}\|_*^2 \|\vv{t} \|_2^2
\end{align*}
and therefore
\begin{align*}
    \sup_{\vv{t} \in T} \phi(\bar{\param},\vv{t}) = \|\bar{\param}\|_*^2 \sup_{\vv{t} \in T} \|\vv{t} \|_2^2 = \|\bar{\param}\|_*^2 n \delta^2
\end{align*}
with the supremum attained only at $t_i = \delta$ due to a norm inequality (recall that $t_i \geq 0$ in $T$). That is, $T(\bar{\param})$ equals the singleton $\bar{\vv{t}}$ with entries $\bar{t}_i = \delta$.\footnote{This argument holds provided $p>2$. For $p=2$, the set $T(\bar{\param})$ has more than one element since the supremum $\sup_{t \in T} \|t \|_2^2$ has multiple solutions, which is why we consider the case $p=2$ separately.} Hence, \eqref{eq_subgrad_robust_risk_small_delta} yields
\begin{align*}
    \partial_{\param} n V_{\delta}(\bar{\param}) =  \partial_{\param} \phi(\bar{\param},\bar{\vv{t}}) =  
    2 \sum_{i=1}^n \delta \|\bar{\param}\|_* (\x_i \mathbb{I}+\delta \partial_{\param} \|\bar{\param}\|_*).
\end{align*}
Pick an arbitrary point in this set:
\begin{equation}\label{eq_min_norm_subgrad_parameter_form}
    2 \sum_{i=1}^n \delta \|\bar{\param}\|_* (\x_i \rho_i+\delta \vv{g_i}),
\end{equation}
where $\rho_i \in \mathbb{I} = [-1,1]$ and $\vv{g_i} \in \partial_{\param} \|\bar{\param}\|_*$. Since $\partial_{\param} \|\bar{\param}\|_*$ is a convex set, we have $\sum_{i=1}^n \vv{g_i} = n\vv{g}$ for some $\vv{g} \in \partial_{\param} \|\bar{\param}\|_*$, which simplifies \eqref{eq_min_norm_subgrad_parameter_form} to
\begin{align*}
    2 \delta \|\bar{\param}\|_* \left [ \sum_{i=1}^n (\x_i \rho_i)+\delta n \vv{g}  \right ] = 
    2 \delta \|\bar{\param}\|_* \left [ X^{\top} \vv{\rho}+\delta n \vv{g}  \right ].
\end{align*}
Thus, $\vv{0} \in \partial_{\param} nV_{\delta}(\bar{\param})$ if and only if 
\begin{equation*}
    \vv{g} = -\frac{X^{\top} \vv{\rho}}{n \delta}
\end{equation*}
for some $\vv{g} \in \partial_{\param} \|\bar{\param}\|_*$ and $\|\vv{\rho}\|_\infty \leq 1$. By \Cref{lemma_when_beta_minimum_norm}, this is if and only if there exists some $\|\vv{\rho}\|_\infty \leq 1$ such that
\begin{equation*}
    \widehat{\vv{\alpha}} := -\frac{\vv{\rho}}{n \delta}
\end{equation*}
solves $\max_{\|X^{\top} \vv{\alpha}\| \leq 1} \vv{\alpha}^{\top} \y$. Equivalently, there exists $\widehat{\vv{\alpha}}$ solving $\max_{\|X^{\top} \vv{\alpha}\| \leq 1} \vv{\alpha}^{\top} \y$ such that
\begin{equation*}
    \|\widehat{\vv{\alpha}} \|_\infty \leq \frac{1}{n \delta}.
\end{equation*}
Crucially, this condition holds if and only if 
\begin{equation*}
    \|\vv{\bar{\alpha}} \|_\infty \leq \frac{1}{n \delta}
\end{equation*}
where $\vv{\bar{\alpha}}$ solves $\max_{\|X^{\top} \vv{\alpha}\| \leq 1} \vv{\alpha}^{\top} \y$ and has the \emph{minimum} $\infty$-norm among all such solutions. In other words, $\vv{\bar{\alpha}}$ solves $\min_{\vv{\alpha}  \in Q} \|\vv{\alpha}\|_\infty$ with $Q = \argmax_{\|X^{\top} \vv{\alpha}\| \leq 1} \vv{\alpha}^{\top} \y$, completing the proof for $2  < p \leq \infty$.

We now continue with the case $p=2$. By \Cref{th_special_cases}, we have that 
\begin{equation*}
    n V_\delta(\param) = \left ( \sqrt{\frac{1}{n}\sum_{i=1}^n (\x_i^{\top} \param - y_i)^2}+\delta \|\param\|_* \right)^2 = \left ( \frac{1}{\sqrt{n}} \|X \param-\y  \|_2 + \delta \|\param\|_* \right)^2
\end{equation*}
To ease the computation of subgradients, we consider instead the equivalent objective
\begin{equation*}
    \sqrt{n V_\delta(\param)} = \frac{1}{\sqrt{n}} \|X \param-\y  \|_2 + \delta \|\param\|_*.
\end{equation*}
Indeed, $\bar{\param}$ minimizes $n V_\delta(\param)$ if and only if $\bar{\param}$ minimizes $\sqrt{n V_\delta(\param)}$, which in turn is if and only if $\vv{0} \in \partial_{\param} \sqrt{n V_\delta(\vv{\bar{\param}})}$. We compute $ \partial_{\param} \sqrt{n V_\delta(\vv{\bar{\param}})}$:
\begin{equation}\label{eq_subgrad_p2_interpolation_root}
     \partial_{\param} \sqrt{n V_\delta(\bar{\param})} = \frac{1}{\sqrt{n}} \partial_{\param}\|X \bar{\param}-\y  \|_2 + \delta \partial_{\param}\|\bar{\param}\|_*.
\end{equation}
By the chain rule for affine transformations (see, e.g., \cite[Proposition 5.4.5]{bertsekas2009convex}), we have that
\begin{equation*}
    \partial_{\param}\|X \bar{\param}-\y  \|_2 = X^T \partial_{\vv{r}} \|\vv{r}\|_2, 
\end{equation*}
where $\partial_{\vv{r}} \|\vv{r}\|_2$ is the subgradient of the norm $\| \cdot \|_2$ at $\vv{r} = X \bar{\param}-y$. Since $\vv{r} = X \bar{\param}-\y = \vv{0}$, we get $\partial_{\vv{r}} \|\vv{r}\|_2 = \{\vv{u}: \|\vv{u} \|_2 \leq 1\}$ and therefore
\begin{equation*}
    \partial_{\param}\|X \bar{\param}-\y  \|_2 = \{X^{\top} \vv{u}: \|\vv{u} \|_2 \leq 1\}.
\end{equation*}
Thus, the subgradient in \eqref{eq_subgrad_p2_interpolation_root} contains zero  if and only if
\begin{equation*}
    \frac{1}{\sqrt{n}} X^{\top} \vv{u} + \delta \vv{g} = \vv{0} \quad \Leftrightarrow \quad \vv{g} = \frac{1}{\delta \sqrt{n}} X^{\top} \vv{u}
\end{equation*}
for some $\|\vv{u}\|_2 \leq 1$ and $\vv{g} \in \partial_{\param}\|\bar{\param}\|_*$. By \Cref{lemma_when_beta_minimum_norm}, this is if and only if there exists $\vv{u} \in \R^n$ such that $\|\vv{u} \|_2 \leq 1$ and
\begin{equation*}
    \widehat{\vv{\alpha}} := -\frac{\vv{u}}{\delta \sqrt{n}}
\end{equation*}
solves $\max_{\|X^{\top} \vv{\alpha}\| \leq 1} \vv{\alpha}^{\top} \y$. Equivalently, there exists $\widehat{\vv{\alpha}}$ solving $\max_{\|X^{\top} \vv{\alpha}\| \leq 1} \vv{\alpha}^{\top} \y$ such that
\begin{equation*}
    \|\widehat{\vv{\alpha}}\|_2 \leq \frac{1}{\delta \sqrt{n}}.
\end{equation*}
Similarly to the case $2 <p \leq \infty$, this holds if and only if  
\begin{equation*}
    \|\vv{\bar{\alpha}}\|_2 \leq \frac{1}{\delta \sqrt{n}},
\end{equation*}
where $\vv{\bar{\alpha}}$ solves $\max_{\|X^{\top} \vv{\alpha}\| \leq 1} \vv{\alpha}^{\top} \y$ and has the \emph{minimum} $2$-norm among all such solutions. In other words, $\vv{\bar{\alpha}}$ solves $\min_{\vv{\alpha}  \in Q} \|\vv{\alpha}\|_2$ with $Q = \argmax_{\|X^{\top} \vv{\alpha}\| \leq 1} \vv{\alpha}^{\top} \y$, completing the proof for $p = 2$. Combining both cases completes the proof in general.
\end{proof} 
\subsection{Computing \texorpdfstring{$\delta_S$}{δ\_S} via a convex program} 
We conclude by showing that $\delta_S$ in \Cref{th_when_beta_minimum_norm} is easy to compute via a convex~program:
\begin{remark}[Computing $\delta_S$]\label{remark_computing_delta_S}
For brevity, let $\bar{\vv{\alpha}}$ denote any solution of $\min_{\vv{\alpha}  \in Q}\|\vv{\alpha}\|_{\infty}$ for $2<p\leq \infty$ and $\min_{\vv{\alpha}  \in Q}\|\vv{\alpha}\|_{2}$ for $p=2$. To compute the bound $\delta_S$ in \Cref{th_when_beta_minimum_norm}, we need to compute $\bar{\vv{\alpha}}$ (any will do). Fortunately, $\bar{\vv{\alpha}}$ can be easily obtained from a simple convex program. Indeed, solve the inner maximization $v^{\star}:= \max_{\|X^{\top} \vv{\alpha}\| \leq 1} \vv{\alpha}^{\top} \y$ first, which is a convex program with linear objective and a compact convex constraint $\|X^{\top} \vv{\alpha}\| \leq 1$ since $X^{\top}$ is injective ($X$ has full rank). Then solve the outer minimization $\min_{\vv{\alpha}  \in Q} \|\vv{\alpha}\|_\infty$ or $\min_{\vv{\alpha}  \in Q} \|\vv{\alpha}\|_2$, with convex and affine constraints $Q = \{\vv{\alpha}: \|X^{\top} \vv{\alpha}\| \leq 1, \vv{\alpha}^{\top} \y = v^{\star} \}$ to obtain $\bar{\vv{\alpha}}$. Furthermore, for $\y\neq0$, we have a unique solution for the inner maximization provided the set $\{\vv{\alpha}:\|X^{\top} \vv{\alpha}\| \leq 1\}$ is strictly convex.\footnote{Indeed, the set $\{ \vv{\alpha}: \|X^{\top} \vv{\alpha}\| \leq 1\}$ contains a ball around the origin and thus spans the directions where $\vv{\alpha}^{\top} \y$ increases. Thus, strict convexity of $\{ \vv{\alpha}: \|X^{\top} \vv{\alpha}\| \leq 1\}$ guarantees uniqueness of $\max_{\|X^{\top} \vv{\alpha}\| \leq 1} \vv{\alpha}^{\top} \y$.} In this case, $\bar{\vv{\alpha}}$ equals this unique solution (for any $p$). An example is the Euclidean norm $\| \cdot \| = \| \cdot \|_2$ where an easy derivation yields $\bar{\vv{\alpha}} = (XX^T)^{-1}\y/\sqrt{\y^{\top} (XX^T)^{-1}\y}$ (with $XX^T$ invertible since $X$ has full rank).\footnote{To see this, let $M := XX^T$. By Cauchy–Schwarz $\vv{\alpha}^\top \y = (X^{\top} \vv{\alpha})^{\top} (X^{\top} M^{-1} \y) \leq \|X^{\top} \vv{\alpha}\|_2 \|X^{\top} M^{-1} \y\|_2$ with equality iff $X^{\top} \vv{\alpha} =c X^{\top} M^{-1} \y$ for some $c>0$. The constraint $\|X^{\top} \vv{\alpha}\|_2 = 1$ (attained at optimality) yields $c = 1/\|X^{\top} M^{-1} \y\|_2$ and thus $X^{\top} \vv{\alpha} =c X^{\top} M^{-1} \y$ iff $\vv{\alpha} =cM^{-1} \y =M^{-1} \y/\|X^{\top} M^{-1} \y\|_2$ (since $X^{\top}$ is injective), completing the derivation.}
\end{remark}

\section{Numerical solvers: Additional details}
\label{efficient_solvers_appendix}

In this section, we provide additional details for the numerical solvers described in \Cref{efficient_solvers}. More precisely, we present two solvers that minimize the robust risk, the saddle-point solver (Appendix \ref{appendix_saddle_point_solver}) and $\eta$-trick solver (Appendix \ref{appendix_eta_trick_solver}), and one solver that instead evaluates the robust risk quickly, by exploiting the one-dimensional subproblems (Appendix \ref{app:sec:rob_eval}). 

\subsection{Saddle-point formulation and solver} \label{appendix_saddle_point_solver}
We use the following saddle-point formulation of the robust risk:

\begin{theorem}[Saddle-point formulation of robust risk]\label{th:saddle_point_form}
The robust risk of Wasserstein DRO linear regression ($2 \leq p < \infty$) with square loss equals
\begin{equation}\label{eq:saddle_point_form}
    n V_{\delta}(\param) = \sup_{\vv{\alpha} \geq 0} K(\param,\vv{\alpha}) :=  \sup_{\vv{\alpha} \geq 0} \left [ n^{1/p} \delta \|\param\|_* \|\vv{\alpha} \|_q+\sum_{i=1}^n \alpha_i |\x_i^T \param - y_i | - \alpha_i^2/4 \right ].
\end{equation}
Moreover, by applying the bijection $\gamma_i = \alpha_i^q$, where $1/p+1/q=1$, we get the following form:
\begin{equation}\label{eq:saddle_point_form_alternative}
n V_{\delta}(\param) = \sup_{\vv{\gamma} \geq 0} K(\param,\vv{\gamma}^{1/q}) =
\sup_{\vv{\gamma} \geq 0} \left [ n^{1/p} \delta \|\param\|_* \|\vv{\gamma}\|_1^{1/q}+\sum_{i=1}^n \gamma^{1/q}_i |\x_i^T \param - y_i | - \gamma_i^{2/q}/4 \right ]
\end{equation}
Crucially, $K(\param,\vv{\gamma}^{1/q})$ is a convex-concave function with respect to $(\param,\gamma)$. Thus, minimizing the robust risk is equivalent to a convex-concave optimization problem $\min_{\param} n V_{\delta}(\param) = \min_{\param} \sup_{\vv{\gamma} \geq 0} H(\param,\vv{\gamma})$, where $H(\param,\vv{\gamma}) := K(\param,\vv{\gamma}^{1/q})$. 
\end{theorem}

The concave-convex formulation allows us to write 
\[
\min_{\param} n V_{\delta}(\param) = \min_{\param} \sup_{\vv{\gamma} \geq 0} H(\param,\vv{\gamma}),
\] 
and we solve the problem using disciplined saddle-point programming \citep{schiele2024tmlr-disciplined}. We also use this result in the proof of \Cref{th_special_cases} in Appendix \ref{proof_th_special_cases}.

\begin{proof}
    We first note that we do not need to prove \eqref{eq:saddle_point_form_alternative}. Indeed, \eqref{eq:saddle_point_form_alternative} follows from \eqref{eq:saddle_point_form} by using the bijection $\gamma_i = \alpha_i^q$. Moreover, $\|\vv{\gamma}\|_1^{1/q}$, $\gamma^{1/q}_i$ and $- \gamma_i^{2/q}$ are all concave in $\vv{\gamma}$ since $1 < q \leq 2$, so $K(\param,\vv{\gamma}^{1/q})$ is concave in $\vv{\gamma} \geq 0$. Finally, $K(\param,\vv{\alpha})$ is convex in $\param$ (since $\| \param \|_*$ and $|\x_i^T \param - y_i |$ are convex), so $K(\param,\vv{\gamma}^{1/q})$ is also convex in $\param$. Thus, $K(\param,\vv{\gamma}^{1/q})$ is a convex-concave function with respect to $(\param,\vv{\gamma})$. Therefore, what remains to prove is \eqref{eq:saddle_point_form}. \Cref{th_supremum_robust_risk}~yields
    \begin{align*}
        n V_{\delta}(\param) = \max_{\vv{t} \geq 0, \|t\|_p^p \leq n\delta^p} \sum_{i=1}^n \Big (\|\param\|_* t_i+ |\x_i^{\top} \param - y_i | \Big)^2 = \\
        \max_{\vv{t} \geq 0, \|t\|_p \leq 1} \sum_{i=1}^n \Big (n^{1/p} \delta \|\param\|_* t_i+ |\x_i^{\top} \param - y_i | \Big)^2 = \\
        \max_{\vv{t} \geq 0, \|t\|_p \leq 1} \left [ \sum_{i=1}^n \sup_{\alpha_i \geq 0} \ \Big (n^{1/p} \delta \|\param\|_* t_i+ |\x_i^{\top} \param - y_i | \Big) \alpha_i - \alpha_i^2/4 \right ] =  \\
        \max_{\vv{t} \geq 0, \|\vv{t}\|_p \leq 1} \sup_{\alpha \geq 0} \left [ n^{1/p} \delta \|\param\|_* \vv{t}^{\top}\alpha+\sum_{i=1}^n \alpha_i |\x_i^{\top} \param - y_i | - \alpha_i^2/4 \right ] = \\
        \sup_{\alpha \geq 0} \left [ n^{1/p} \delta \|\param\|_*  \max_{t \geq 0, \|\vv{t}\|_p \leq 1} \{\vv{t}^{\top}\alpha \}+\sum_{i=1}^n \alpha_i |\x_i^{\top} \param - y_i | - \alpha_i^2/4 \right ] = \\
        \sup_{\vv{\alpha} \geq 0} \left [ n^{1/p} \delta \|\param\|_* \|\vv{\alpha}\|_q+\sum_{i=1}^n \alpha_i |\x_i^{\top} \param - y_i | - \alpha_i^2/4 \right ] = \sup_{\vv{\alpha} \geq 0} K(\param,\vv{\alpha})
    \end{align*}
    where we used that $\|\vv{\alpha}\|_q = \max_{\|\vv{t}\|_p \leq 1} \vv{t}^{\top} \vv{\alpha} = \max_{\vv{t} \geq 0, \|\vv{t}\|_p \leq 1} \vv{t}^{\top} \vv{\alpha}$ since $\vv{\alpha} \geq 0$. This proves \eqref{eq:saddle_point_form} and completes the proof.
\end{proof}
\subsection{\texorpdfstring{$\eta$}{η}-trick solver}\label{appendix_eta_trick_solver} 

In this section, we describe in detail the $\eta$-trick solver that iteratively alternative between solving a weighted ridge regression problem and updating the weights in closed-form. We take inspiration from previous works on adversarial linear regression~\citep{pmlr-v258-ribeiro25a, ribeiro_kernel_2025} and use what is sometimes referred to as the ``$\eta$-trick'' (see, e.g.,~\citealt{Itrick019}) to minimize the risk $V_\delta(\param)$ with respect to $\param$. We only consider the case where we have the $\infty$-norm in the Wasserstein cost, i.e., $\lVert \cdot \rVert = \lVert \cdot \rVert_\infty$ and $\lVert \cdot \rVert_* = \lVert \cdot \rVert_1$. We leave the extension to other norms as future~work.

The $\eta$-trick refers to a family of variational identities that are useful for optimization \citep{Itrick019}. For the squared $1$-norm:
\begin{equation*}
    \lVert \w \rVert_1^2 = \inf_{\boldeta \in \Delta_d} \sum_{j=1}^d \frac{w_j^2}{\eta_j},
\end{equation*}
where $\Delta_d = \{ \boldeta \in \mathbb{R}_{>0}^d : \sum_{j=1}^d \eta_j = 1 \}$ is the $d$-dimensional simplex and $\R_{>0} = (0,\infty)$. Now, consider the formulation of the risk in Theorem~\ref{th_supremum_robust_risk}
\begin{equation}\label{app:eq:th1risk}
V_{\delta}(\param) = \sup_{\vv{t} \in T} \frac{1}{n} \sum_{i=1}^n (r_i+t_i \|\param\|_1)^2 \quad \text{with} \quad T = \{ t \in \mathbb{R}_{+}^n : \lVert \vv{t} \rVert_p \leq n^{1/p} \delta \}, 
\end{equation}
where $\R_{+} = [0,\infty)$, $r_i = |\x_i^{\top}\param-y_i|$ and we used that $\lVert \cdot \rVert_* = \lVert \cdot \rVert_1$. Letting $\w^{(i)} = [r_i, t_i \beta_1, \dots, t_i \beta_d]^\top$ and applying the trick then gives
\[
(r_i + t_i \|\param\|_1)^2 = \lVert \w^{(i)} \rVert_1^2 = \inf_{\boldeta^{(i)} \in \Delta_{d+1}} \left\{ \frac{r_i^2}{\eta_0^{(i)}} + t_i^2 \sum_{j=1}^d \frac{ \beta_j^2}{\eta_j^{(i)}} \right\}.
\]
Plugging this into \eqref{app:eq:th1risk} we get
\begin{equation*}
     V_\delta(\param) = \sup_{\vv{t} \in T} \frac{1}{n} \sum_{i=1}^n \inf_{\boldeta^{(i)} \in \Delta_{d+1}} \left\{ \frac{r_i^2}{\eta_0^{(i)}} + t_i^2 \sum_{j=1}^d \frac{ \beta_j^2}{\eta_j^{(i)}} \right\},
\end{equation*}
where we can swap the summation and the infimum since the problem is separable over $i$:
\begin{equation}\label{app:eq:supmin}
     V_\delta(\param) = \sup_{\vv{t} \in T} \inf_{\boldeta \in \Delta_{d+1}^n} \underbrace{\frac{1}{n} \sum_{i=1}^n  \left\{ \frac{r_i^2}{\eta_0^{(i)}} + t_i^2 \sum_{j=1}^d \frac{ \beta_j^2}{\eta_j^{(i)}} \right\}}_{\Phi(\param, \vv{t}, \boldeta)}.
\end{equation}
Here, we denote $\boldeta := [\boldeta^{(1)}, \dots, \boldeta^{(n)}]^\top \in \Delta_{d+1}^n$ and $\Phi$ the current objective function. If we make the substitution $z_i = t_i^2$, we get
\begin{equation*}
    \tilde{\Phi}(\param, \vv{z}, \boldeta) = \frac{1}{n} \sum_{i=1}^n  \left\{ \frac{r_i^2}{\eta_0^{(i)}} + z_i \sum_{j=1}^d \frac{ \beta_j^2}{\eta_j^{(i)}} \right\} \quad \text{and} \quad Z = \{ z \in \mathbb{R}_+^n : \lVert \vv{z} \rVert_{p/2} \leq n^{2/p} \delta^2 \}.
\end{equation*}
Then, 
\begin{itemize}
    \item $\tilde{\Phi}(\param, \cdot, \boldeta)$ is continuous and affine over $Z$ for every fixed $(\param, \boldeta) \in \mathbb{R}^d \times \Delta_{d+1}^n$.
    \item $\tilde{\Phi}(\param, \vv{z}, \cdot)$ is continuous and convex over $\Delta_{d+1}^n$ for every fixed $(\param, \vv{z}) \in \mathbb{R}^d \times Z$.
    \item $Z$ is compact convex since $p\geq2$. 
    \item $\Delta_{d+1}^n$ is convex.
\end{itemize}
Thus, Sion's minimax theorem \citep{sion_minimax_1958, hidetoshi1988} applies\footnote{We use the version for example found in \citet[Theorem 28.12]{Lattimore_Szepesvári_2020} where either $Z$ or $\Delta_{d+1}^n$ needs to be compact, see also the original Corollary 3.3 in \citet{sion_minimax_1958}.} and we can swap supremum and infimum in \eqref{app:eq:supmin}:
\begin{equation}\label{app:eq:minsup}
     V_\delta(\param) = \inf_{\boldeta \in \Delta_{d+1}^n} \sup_{\vv{z} \in Z} \frac{1}{n} \sum_{i=1}^n  \left\{ \frac{r_i^2}{\eta_0^{(i)}} + z_i \sum_{j=1}^d \frac{ \beta_j^2}{\eta_j^{(i)}} \right\}.
\end{equation}
The supremum can now be solved in closed form. Let $u_i := \sum_{j=1}^d \beta_j^2 / \eta_j^{(i)}$ so that
\begin{equation*}
     \sup_{\vv{z} \in Z} \frac{1}{n} \sum_{i=1}^n  \left\{ \frac{r_i^2}{\eta_0^{(i)}} + z_i \sum_{j=1}^d \frac{ \beta_j^2}{\eta_j^{(i)}} \right\} 
     =  \frac{1}{n} \sum_{i=1}^n  \frac{r_i^2}{\eta_0^{(i)}} + \frac{1}{n} \sup_{\vv{z} \in Z} \vv{z}^\top \vv{u}.
\end{equation*}
Then, by letting $k(p)$ be the conjugate exponent of $p/2$, i.e.,
\[
k(p) = 
\begin{cases}
    \infty & p=2 \\
    p/(p-2) & 2 < p < \infty \\
    1 & p = \infty
\end{cases}
\qquad \text{so that}
\qquad \frac{1}{p/2} + \frac{1}{k(p)} = 1, \; \forall p \in [2, \infty],
\]
we have that
\[
\sup_{\vv{z} \in Z} \vv{z}^\top \vv{u} = n^{2/p} \delta^2 \lVert \vv{u} \rVert_{k(p)}
\]
by the dual-norm identity since $u_i \geq 0$. Thus, we can define
\begin{equation*}
    G_\delta(\param, \boldeta) := \frac{1}{n} \sum_{i=1}^n  \frac{r_i^2}{\eta_0^{(i)}} + n^{2/p - 1} \delta^2 \lVert \vv{u} \rVert_{k(p)},
\end{equation*}
and therefore,
\begin{equation}\label{app:eq:G}
    V_\delta(\param) = \inf_{\boldeta \in \Delta_{d+1}^n} G_\delta(\param, \boldeta).
\end{equation}
We note here that \(G_\delta\) is jointly convex in \((\param,\boldeta)\) since the quadratic-over-linear function $(a,b) \mapsto a^2/b$ is jointly convex for $b>0$ \cite[Chapter 3.1.5]{Boyd_Vandenberghe_2004_1} and since the norm preserves the convexity. Hence, we can alternate between minimization over $\param$ and minimization over $\boldeta$, particularly the optimal $\boldeta$ has a structure that turns the subsequent $\param$-update into a weighted ridge regression problem. The $\boldeta$-subproblem itself can be solved numerically (for example, by using CVXPY), but we can also recover its solution explicitly from the maximizer of the supremum in Theorem~\ref{th_supremum_robust_risk}.

To derive this update for $\delta > 0$ and nonzero residuals $r_i \neq 0$, fix $\param \neq 0$ and from \eqref{app:eq:th1risk} let $\vv t^\star$ maximize 
\[
\frac1n\sum_{i=1}^n(r_i+t_i \lVert \param \rVert_1  )^2
\qquad\text{over }\vv t\in T.
\]
Finding $\vv t^\star$ is the same problem as evaluating $V_\delta(\param)$ and can be solved using the scalar reduction derived in Appendix \ref{app:sec:rob_eval}. We now show how $\vv t^\star$ allows us to recover a minimizer of $G_\delta(\param,\cdot)$. From \eqref{app:eq:supmin}, recall that
$\vv t^\star$ maximizes $\Phi(\param, \vv{t}, \boldeta)$ after minimization
over $\boldeta$:
\[
    \vv t^\star \in
    \arg\max_{\vv{t} \in T}
    \inf_{\boldeta\in\Delta_{d+1}^n}
    \Phi(\param,\vv t,\boldeta).
\]
That is, \eqref{app:eq:supmin} states
\begin{equation}\label{app:eq:Phi}
    \inf_{\boldeta\in\Delta_{d+1}^n}
    \Phi(\param,\vv t^\star,\boldeta)
    = V_\delta(\param).
\end{equation}
On the other hand, $G_\delta$ was obtained by solving
the supremum in \eqref{app:eq:minsup}:
\[
    G_\delta(\param,\boldeta)
    =
    \sup_{\vv z\in Z}\tilde{\Phi}(\param,\vv z,\boldeta)
    =
    \sup_{\vv t\in T}\Phi(\param,\vv t,\boldeta)
\]
for an \emph{arbitrary} $\boldeta$. Thus, $\vv t^\star$ is not guaranteed to attain this supremum,
but its feasibility gives
\[
    G_\delta(\param,\boldeta)
    \ge \Phi(\param,\vv t^\star,\boldeta).
\]
Minimizing over $\boldeta$ on both sides give
\[
    V_\delta(\param) = \inf_{\boldeta \in \Delta_{d+1}^n} G_\delta(\param,\boldeta)
    \geq \inf_{\boldeta \in \Delta_{d+1}^n} \Phi(\param,\vv t^\star,\boldeta) = V_\delta(\param),
\]
where we have used \eqref{app:eq:G} and \eqref{app:eq:Phi}, respectively. We use this to execute the $\boldeta$-update. More precisely, for a fixed current parameter $\widetilde{\param}$ (in our algorithm), the $\boldeta$-update is thus the minimizer $\hat{\boldeta}$ of
\begin{equation}\label{app:eq:Phi_tstar}
    \Phi(\widetilde{\param}, \vv{t}^\star, \boldeta) = \frac{1}{n} \sum_{i=1}^n  \left\{ \frac{\tilde{r}_i^2}{\eta_0^{(i)}} + {(t_i^\star)}^2 \sum_{j=1}^d \frac{ \widehat{\beta}_j^2}{\eta_j^{(i)}} \right\},
\end{equation}
where $\tilde{r}_i = |\x_i^{\top}\widetilde{\param}-y_i|$. From \citet[Proposition 4]{pmlr-v258-ribeiro25a}, this minimizer $\hat{\boldeta}$ is uniquely determined and given by
\begin{equation}
\hat{\eta}_0^{(i)} := \frac{\tilde{r}_i}{\tilde{r}_i + t_i^\star \lVert \widehat{\param} \rVert_1},
\qquad
\hat{\eta}_j^{(i)}:=\frac{t_i^\star|\widehat{\beta}_j|}{\tilde{r}_i + t_i^\star \lVert \widehat{\param} \rVert_1},
\quad j=1,\ldots,d.
\end{equation}
This completes the $\boldeta$-update. Next, we want to execute the $\param$-update (since we alternate), with updated parameter denoted by $\param^+$. To obtain $\param^+$ from the $\param$-update, we plug $\hat{\boldeta}$ back into \eqref{app:eq:Phi_tstar}, which gives
\begin{align}
    \Phi(\param, \vv{t}^\star, \hat{\boldeta}) 
    &= \frac{1}{n} \sum_{i=1}^n  \left\{ r_i^2 \frac{\tilde{r}_i + t_i^\star \lVert \widetilde{\param} \rVert_1}{\tilde{r}_i} + {(t_i^\star)}^2 \sum_{j=1}^d \beta_j^2 \frac{\tilde{r}_i + t_i^\star \lVert \widetilde{\param} \rVert_1}{t_i^\star|\widetilde{\beta}_j|} \right\} \\
    &= \frac{1}{n} \sum_{i=1}^n \hat{w}_i r_i^2 + \sum_{j=1}^d \hat{\gamma}_j \beta_j^2,
\end{align}
where
\begin{equation}
    \hat{w}_i = \frac{\tilde{r}_i + t_i^\star \lVert \widetilde{\param} \rVert_1}{\tilde{r}_i} \quad \text{and} \quad \hat{\gamma}_j = \left[ \frac{1}{n} \sum_{i=1}^n t_i^\star(\tilde{r}_i + t_i^\star \lVert \widetilde{\param} \rVert_1) \right] \frac{1}{\lvert \widetilde{\beta}_j \rvert}, \quad \forall (i,j) \in [n] \times [d],
\end{equation}
using the shorthand $[n] = \{1,2,\dots,n\}$, and same for $[d]$. We here note that minimizing $G_\delta(\cdot,\hat{\boldeta})$ (and therefore obtain $\param^+$) is now a weighted ridge regression problem
with $\widehat{W}=\mathrm{diag}(\hat{w}_1,\ldots,\hat{w}_n)$ and
$\widehat{\Gamma}=\mathrm{diag}(\hat{\gamma}_1,\ldots,\hat{\gamma}_d)$.
The $\param$-update is therefore obtained by solving
\[
    (X^\top \widehat{W} X+n\widehat{\Gamma})\param^+
    =X^\top \widehat{W}y
\]
and assigning the solution $\param^+$ as our new $\widetilde{\param}$, from which we again perform an $\boldeta$-update, and so on. This alternating scheme is repeated until a stopping
criterion is satisfied. We initialize with uniform weights $\eta_j^{(i)}=1/(d+1)$.
Then $u_i=(d+1)\|\param\|_2^2$, so
$\|\vv u\|_{k(p)}=n^{1/k(p)}(d+1)\|\param\|_2^2$.
Since $2/p+1/k(p)=1$, substitution gives
\[
G_\delta(\param,\boldeta_{\mathrm{unif}})
=(d+1)\left[
\frac1n\|X\param-\y\|_2^2+\delta^2\|\param\|_2^2
\right].
\]
Thus, the initial $\param$-update solves
$(X^\top X+n\delta^2 I)\param=X^\top\y$. We summarize this algorithm by Algorithm \ref{alg:robust-ridge}, where we also add smoothing for numerical stability. The smoothing adds only terms that are constant in $\param$ and therefore leaves this initialization unchanged. Moreover, $\mathtt{ScalarReduction}$ in Algorithm \ref{alg:robust-ridge} corresponds to solving the scalar form in \Cref{lemma_robust_risk_1D_formulation}, described in Appendix \ref{app:sec:rob_eval}. We use this form since it is easier to solve it than the robust quadratic form given by \Cref{th_supremum_robust_risk}. Note that $\vv{t}^\star$ is the same for both these forms due to strong duality, see, e.g., the proof of \Cref{th_supremum_robust_risk}.

\begin{algorithm}[t]
\caption{$\eta$-trick solver}
\label{alg:robust-ridge}
\KwInput{$X, \y; \quad \delta \geq 0; \quad p\in[2,\infty]; \quad \epsilon>0; \quad \text{tolerance } \tau>0; \quad \text{iterations }M$}
\KwInitialize{solve $(X^\top X+n\delta^2 I)\param=X^\top \y \textnormal{ (uniform } \boldeta \textnormal{ init})$}

\If{$\delta=0$}{\KwRet{$\param$}}
\For{$m=1,\ldots,M$}{
    \Comment{smoothed magnitudes and worst-case transport}
    $r_i\gets\sqrt{(\x_i^\top\param-y_i)^2+\epsilon^2}$,
    \quad $b_j\gets\sqrt{\beta_j^2+\epsilon^2}$\;
    $B\gets\sum_j b_j$\;
    $\vv t^\star\gets
    \mathtt{ScalarReduction}(\vv r,B,\delta,p)$\;
    \vspace{\baselineskip}
    \Comment{compute weights and solve weighted ridge}
    $w_i\gets (r_i+t_i^\star B)/r_i$,
    \quad $\gamma_j\gets [\frac{1}{n} \sum_i t_i^\star (r_i+t_i^\star B)] / b_j$\;
    $W\gets\operatorname{diag}(\vv w)$,
    \quad $\Gamma\gets\operatorname{diag}(\boldsymbol\gamma)$\;
    \textit{solve} $(X^\top W X+n\Gamma)\param^+=X^\top W \y$\;
    \vspace{\baselineskip}
    \Comment{update and check stopping criterion}
    $\text{change} \gets \|\param^+ - \param\|_2 / (1+\|\param\|_2) $\;
    $\param\gets\param^+$\;
    \If{$ \textnormal{change} \leq\tau$}{\textbf{break}\;}
}
\KwRet{$\param$}
\end{algorithm}

\subsection{Fast robust risk evaluations}\label{app:sec:rob_eval}
To evaluate only the robust risk $V_{\delta}(\param)$ (rather than minimize it), we can compute it faster by using \Cref{lemma_robust_risk_1D_formulation}. First, the special cases $\delta=0$, $p=2$, and $p=\infty$ have closed-form expressions (described above for minimizing the robust risk) and can be readily computed. The same holds for $\param = 0$. We may therefore assume $\delta >0$ and $2<p<\infty$ and $\param \neq 0$. We compute the robust risk for this general case using \Cref{lemma_robust_risk_1D_formulation}. Note first that $\lambda = 0$ is suboptimal in this case since the inner supremum of \eqref{eq_lemma_robust_risk_1D_formulation} becomes infinite. Hence, we may assume $\lambda>0$. Given $\lambda > 0$, we can easily compute the inner supremum $\sup_{t_i \ge 0} \Big\{ \big( \lvert r_i \rvert + t_i \lVert {\param} \rVert_* \big)^2 - \lambda t_i^p \Big\}$ of \eqref{eq_lemma_robust_risk_1D_formulation}. Indeed, let $f_i$ denote the objective function of this supremum:
\[
f_i(t_i) := \big( \lvert r_i \rvert + t_i \lVert {\param} \rVert_* \big)^2 - \lambda t_i^p = t_i^2 \lVert \param \rVert_*^2 + 2 \lvert r_i \rvert t_i \lVert \param \rVert_* + r_i^2 - \lambda t_i^p.
\]
Then,
\[
  f_i'(t_i) = 2t_i \lVert \param \rVert_*^2 + 2 \lvert r_i \rvert \lVert \param \rVert_* - \lambda p t_i^{p-1},
  \qquad
  f_i''(t_i) = 2 \lVert \param \rVert_*^2 - \lambda p (p-1) t_i^{p-2}.
\]
Given $2 <p <\infty$, we have that $f_i''$ is strictly decreasing with $f_i''(0) > 0$, so $f_i'$ first increases and then decreases.
Since $f_i'(0) = 2 \lvert r_i \rvert \lVert \param \rVert_* \ge 0$ and $f_i'(t_i) \to -\infty$, it follows
that $f_i'$ changes sign at most once. Hence $f_i$ is unimodal on $[0,\infty)$, with maximizer either at
a stationary point $t^\star_i$  (unique if it exists) or at the boundary $t_i = 0$. The inner supremum $\sup_{t_i\geq0} f_i(t)$ in \eqref{eq_lemma_robust_risk_1D_formulation} is therefore straightforward to solve. Concretely, we minimize the negated objective using SciPy \citep{virtanen_scipy_2020}, where we first apply the method \texttt{scipy.optimize.elementwise.bracket\_minimum} to locate a small interval on $[0,\infty)$ where the minimum is attained, and then use \texttt{scipy.optimize.elementwise.find\_minimum} to compute the minimum exactly. For further details, see the code implementation.

For the outer problem in \eqref{eq_lemma_robust_risk_1D_formulation}, note that $f_i$ is affine in $\lambda$, and hence $\sup_{t_i\geq0} f_i(t)$ is convex in $\lambda$. Adding $n\delta^p \lambda$ preserves this convexity, so the outer minimization over $\lambda$ is a minimization of a convex function. We solve this outer problem using \texttt{scipy.optimize.minimize\_scalar}, recomputing the inner maximizers for each candidate $\lambda$, to obtain the robust risk $V_{\delta}(\param)$ to numerical accuracy. For further details, see the code implementation.


\section{Numerical experiments: Additional details}
In this section, we provide additional details of the numerical simulations in \Cref{numerical_experiments}, starting with the setup details for the fast rate and slow rate, followed by details of the small and large $\delta>0$ simulations.

\subsection{Fast rate}
In this section, we present additional details for the fast rate simulations, showing that $X$ satisfies the RE condition and providing some additional simulations.

\textbf{$\boldsymbol{X}$ satisfies the RE condition.} We here show that $X$ with entries sampled independently and uniformly in $[-1,1]$ is in $\mathrm{RE}(s,\ell)$ for any $s$ and $\ell$ with $\kappa=1/\sqrt{6}$, which holds with arbitrarily high probability for sufficiently large $n$. The argument is standard and based on Rayleigh–Ritz variational characterization (see, e.g., \cite[Equation (6.3)]{wainwright_high-dimensional_2019}) to obtain a sufficient criterion for satisfying the RE condition. To this end, pick any $s$ and $\ell$, consider an index set $S\subseteq\{1,\ldots,d\}$, and define the cone
$$
\mathcal C_S=\left\{\vv{v}\in\mathbb R^d:
\|\vv{v}_{S^c}\|_1\leq \ell \|\vv{v}_S\|_1\right\}.
$$
Moreover, let $G=X^\top X/n$ and $\lambda_{\min}(G)$ be the smallest eigenvalue of $G$. Using standard results from random matrix theory, we have that $\lambda_{\min}(G)$ is lower bounded by $1/3-\epsilon$ with arbitrarily high probability for any given $\epsilon>0$ and sufficiently large $n$.\footnote{This is seen as follows. Let $\Sigma := \E [\vv{v} \vv{v}^T]$ where $\vv{v} \in \R^d$ has independent entries uniformly sampled in $[-1,1]$. Then, $\Sigma = I/3$ due to independence and variance $\E[v_i^2] = 1/3$. Moreover, by Equation (6.7) in \cite{wainwright_high-dimensional_2019} (Weyl's theorem), we have that $\lambda_{\min}(\Sigma)-\lambda_{\min}(G) \leq \|G-\Sigma\|_2$ and so $\lambda_{\min}(G) \geq \lambda_{\min}(\Sigma) -\|G-\Sigma\|_2 = 1/3-\|G-\Sigma\|_2$. By Theorem 6.5 in \cite{wainwright_high-dimensional_2019} for some universal constants $(c_i)_{i=1}^3$, we have that $\|G-\Sigma\|_2 < c_1  [d/n+\sqrt{d/n}] + \epsilon$ holds for any $\epsilon \geq 0$ with probability $\geq 1-c_2e^{-c_3n \min(\epsilon,\epsilon^2)}$. That is, $\lambda_{\min}(G) \geq 1/3-(c_1[d/n+\sqrt{d/n}] + \epsilon)$ with probability $\geq 1-c_2e^{-c_3n \min(\epsilon,\epsilon^2)}$, which yields the result.} For convenience, pick $\epsilon=1/6$ so that
$$
\kappa = \frac{1}{\sqrt{6}} \leq \sqrt{\max\{\lambda_{\min}(G),0\}}
$$
holds with arbitrarily high probability for large enough $n$. Then, the Rayleigh–Ritz variational characterization \cite[Equation (6.3)]{wainwright_high-dimensional_2019} yields for every $\vv{v}\in\mathbb R^d$:
$$
\frac{\|X\vv{v}\|_2^2}{n}
=\vv{v}^\top \frac{1}{n} X^\top X \vv{v} = \vv{v}^\top G \vv{v}
\geq\lambda_{\min}(G)\|\vv{v}\|_2^2
\geq \kappa^2\|\vv{v}\|_2^2,
$$
which holds with arbitrarily high probability for large enough $n$. Since this is valid for every $\vv{v}\in\mathbb R^d$, it is also valid for every $\vv{v}$ in $\mathcal C_S$. In other words, we have that $X$ is in $\mathrm{RE}(s,\ell)$ for $\kappa=1/\sqrt{6}$ with arbitrarily high probability (for large enough $n$), completing the derivation.

\begin{figure}[t]
\centering
\begin{subfigure}{.47\textwidth}
    \centering
    \includegraphics[width=1\linewidth]{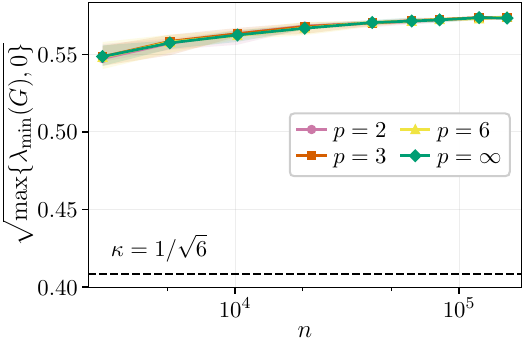}
\end{subfigure}
\quad
\begin{subfigure}{.47\textwidth}
    \centering
    \includegraphics[width=1\linewidth]{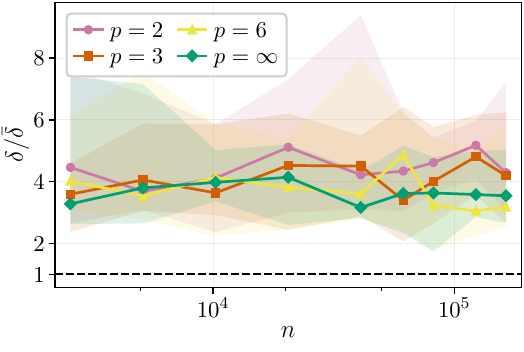}
\end{subfigure}
\caption{Additional plots for the fast rate simulations in \Cref{numerical_experiments} showing that the conditions in \Cref{thm_fast_rate_main_theorem_v2} hold. Concretely, the left plot shows that the RE condition is satisfied for $\kappa = 1/\sqrt{6}$ (for all $n$) since it indeed lower-bounds the minimum eigenvalue $\lambda_{\min}(G)$ of $G=X^TX/n$ (see \eqref{eq_RE_condition_suff_cond} and the corresponding derivation for details). Moreover, the right plot shows that $\delta \geq \bar{\delta}$ (which it should be with high probability). Therefore, the conditions in \Cref{thm_fast_rate_main_theorem_v2} hold. Both plots show median over 10 runs with shaded bands showing minimum and maximum across the repetitions.}
\label{fig_fast_rate_conditions_simulated}
\end{figure}

\begin{figure}[t]
\centering
\includegraphics[width=0.5\textwidth]{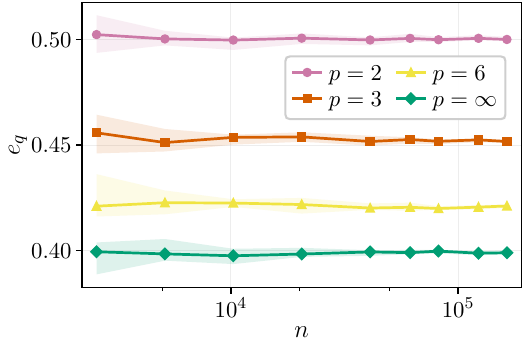}
\caption{Convergence for $e_q := \frac{\|\vv{\varepsilon}\|_q}{n^{1/q}}$ in the fast rate simulations from \Cref{numerical_experiments}. We see that all $e_q$ quickly stabilizes toward a constant as $n$ increases. Since $e_q$ converges to a constant, so do $B$ and $C$ in \Cref{thm_fast_rate_main_theorem_v2}. The plot shows the median over 10 repeats with shaded bands being the 10th-90th percentile range.}
\label{fig_paper_eq}
\end{figure}

\textbf{Additional plots.}
We also provide plots that support the claim that the conditions in \Cref{thm_fast_rate_main_theorem_v2} hold for the fast rate simulations in \Cref{numerical_experiments}, seen in Figure \ref{fig_fast_rate_conditions_simulated}.  Indeed, we see that both the RE condition and that $\delta \geq \bar{\delta}$ are fulfilled. Moreover, we also plot $e_q := \frac{\|\vv{\varepsilon}\|_q}{n^{1/q}}$ in Figure \ref{fig_paper_eq}. We see that as $n$ increases, $e_q$ converges to a constant as predicted. Moreover, since both $B$ and $C$ in \Cref{thm_fast_rate_main_theorem_v2} are functions of $e_q$ that converge to constants if $e_q$ converges to a constant, we have that the maximum in \eqref{eq_fast_rate_main_theorem_v2_bound} becomes a constant. Thus, we indeed get the rate $O(\delta^2)$ which equals $O(n^{-1})$ since $\delta$ decreases as $1/\sqrt{n}$. 

\subsection{Slow rate}
We continue with the slow rate. More precisely, we here provide the details on how $X$ is set and why the RE condition fails for it. First of all, $\param^* = [2, -1]^\top$ is $2$-sparse, so we need an RE condition of the form $X \in \mathrm{RE}(2,\ell)$ for some fixed $\ell$ and $\kappa>0$. We construct $X$ so that there is no such $\ell$ and $\kappa>0$ such that $X \in \mathrm{RE}(2,\ell)$ holds for all $n$. This construction closely resembles \citet{van2017some}; see also references therein. Importantly, we want to construct $X$ such that the entries are \emph{correlated} in a way that breaks the RE condition. We achieve this by setting the first column $\vv{X}_1$ and the second column $\vv{X}_2$ of $X$ to the following:
\begin{equation*}
    \vv{X}_1 =\sqrt{\frac{1+\rho_n}{2}}\vv{u}+\sqrt{\frac{1-\rho_n}{2}} \vv{w}, \quad \vv{X}_2 = \sqrt{\frac{1+\rho_n}{2}} \vv{u}-\sqrt{\frac{1-\rho_n}{2}} \vv{w},
\end{equation*}
where $\vv{u}$ and $\vv{w}$ are any vectors in $\{-1,1\}^n$ such that $\vv{u}$ is orthogonal to $\vv{w}$ for even $n$ (and for odd $n$ set them so that they are ``almost'' orthogonal $\vv{w}^{\top} \vv{u} = \pm 1$), and $\rho_n = 1-c_0/\sqrt{n}$ for some constant $c_0>0$. It is easy to see that the entries of $X$ are uniformly bounded with $M=\sqrt{2}$. Crucially, this construction is made so that for every even $n$,
\begin{equation*}
    G = \frac{1}{n} X^TX =
    \begin{pmatrix}
    1 & \rho_n \\
    \rho_n & 1
    \end{pmatrix}
\end{equation*}
becomes close to a singular matrix. Indeed, note that $\lambda_{\min}(G) = 1-\rho_n = c_0/\sqrt{n}$ is the smallest eigenvalue. Thus, the Rayleigh–Ritz variational characterization \cite[Equation (6.3)]{wainwright_high-dimensional_2019} implies that for any even $n$, there exists a vector $\vv{v} \in \R^2$ such that:
\begin{equation}\label{eq_RE_condition_suff_cond}
\frac{\|X\vv{v}\|_2^2}{n}
=\vv{v}^\top \frac{1}{n} X^\top X \vv{v} = \vv{v}^\top G \vv{v}
=\lambda_{\min}(G)\|\vv{v}\|_2^2 = \frac{c_0}{\sqrt{n}} \|\vv{v}\|_2^2.
\end{equation}
Thus, since $\vv{v}$ is trivially in the cone $\mathcal{C}_{S} = \{\vv{v}\in\mathbb R^2: \|\vv{v}_{S^c}\|_1\leq \ell \|v_S\|_1\} = \R^2$ for $S=\{1,2\}$ and $\lambda_{\min}(G) = \frac{c_0}{\sqrt{n}} \rightarrow 0$ as $n \rightarrow \infty$, we conclude that no $\kappa>0$ exists such that $X \in \mathrm{RE}(2,\ell)$ for large enough $n$, proving the claim. 

Concretely, in the experiments, we set $c_0 =3$ and set $\vv{u}$ and $\vv{w}$ to $n/4$ stacked copies of the vectors $[1,1,-1,-1]^{\top}$ and $[1,-1,1,-1]^T$, respectively (with $n$ a multiple of $4$) and plot only such $n$, namely, $n =4096, 6144, 8192, 10240$. The reason for this is twofold. First, it is easy to implement stacked copies. Second, there is a convenient analytic interpretation. Indeed, since $\vv{u}$ and $\vv{w}$ always have zero mean (their entries sum to zero), the parameter $\rho_n$ equals the standard empirical correlation between the columns $\vv{X}_1$ and $\vv{X}_2$. To see this, consider the sample correlation coefficient (see, e.g., \citet[Equation (3.18)]{james2021introduction}) between $\vv{X}_1$ and $\vv{X}_2$ given~by
\begin{equation*}
r = \frac{\sum_{i=1}^n (X_1[i]-\textrm{mean}(\vv{X}_1)) (X_2[i]-\textrm{mean}(\vv{X}_2))}{\|\vv{X}_1\|_2 \|\vv{X}_2\|_2}
\end{equation*}
where $X_1[i]$ is the $i$-th entry of $\vv{X}_1$ (and similarly for $\vv{X}_2$), and $\textrm{mean}(\vv{X}_i)$ is the mean over all entries of $\vv{X}_i$. In our case, $\textrm{mean}(\vv{X}_i) = 0$ by construction, and a simple computation yields $r = \rho_n$ as desired. Therefore, another interpretation of what happens is as follows. As $n \rightarrow \infty$, we have that $r = \rho_n \rightarrow 1$; in other words, $\vv{X}_1$ and $\vv{X}_2$ become maximally correlated. This correlation breaks down the RE condition and hinders the fast rate $O(n^{-1})$. As a result, we get only the slow rate $O(n^{-1/2})$.


\subsection{Small and large \texorpdfstring{$\delta$}{δ}}
Finally, we describe the setup for the small and large $\delta>0$ simulations. In both cases, we let $y_i=\x_i^{\top} \param^* + \varepsilon_i$ with $n=30$ and $d=60$, true parameter iid sampled as $\param^*_i \sim N(0,1)$, noise $\vv{\varepsilon} \sim N(0,\sigma^2I)$ with $\sigma=0.5$, and entries $x_{ij} \in X$ iid sampled as $x_{ij} \sim N(0,1)$. Note that $d>n$, that is, we are in the overparametrized regime. Moreover, since we sample $x_{ij} \sim N(0,1)$, we have that $X$ is of full row rank almost surely. However, as an extra precaution, we also check that $X$ is indeed of full rank to be sure that the conditions in \Cref{th_when_beta_minimum_norm} hold; For all simulations, $X$ was indeed of full row rank. 

For the small $\delta$ case, we plot the average error $\|X \widehat{\param}-\y\|_2^2/n$ as a function of $\delta$ as seen in Figure \ref{fig_paper_extreme_delta} (upper) for $p \in \{2,3,6,\infty\}$. Moreover, we also compute the threshold $\delta_S$ for $p=2$ (pink dashed) and $\delta_S$ for $2 <p\leq \infty$ (gray dashed, identical for all $2 <p\leq \infty$), as given by \Cref{th_when_beta_minimum_norm}, where $\min_{\vv{\alpha}  \in Q} \|\vv{\alpha}\|_\infty$ and $\min_{\vv{\alpha}  \in Q} \|\vv{\alpha}\|_2$ in the expression for $\delta_S$ are computed using the convex program as detailed by \Cref{remark_computing_delta_S}. For each $p \in \{2,3,6,\infty\}$, we observe that the average error $\|X \widehat{\param}-\y\|_2^2/n$ decreases rapidly for $\delta \leq \delta_S$, reaching values in the numerical precision regime, indicating that interpolation occurs. Furthermore, this interpolation transition happens simultaneously for all $2 <p\leq \infty$ at their $\delta_S$, while it occurs sooner for the $p=2$ case since its threshold $\delta_S$ is larger. All these observations are in line with \Cref{th_when_beta_minimum_norm}. 

For the large $\delta$ case, we plot $\|\widehat{\param}\|_{\infty}$ as a function of $\delta$, as seen in Figure \ref{fig_paper_extreme_delta} (lower) for $p \in \{2,3,5,\infty\}$ with threshold $\delta_L$ for each $p$ (dashed lines with matching color). We see that $\|\widehat{\param}\|_{\infty}$ decreases in the region $\delta \leq \delta_L$ and becomes zero for $\delta \geq \delta_L$ as predicted by \Cref{th_when_beta_zero}. Note also that the threshold $\delta_L$ is different for each $p$ (as opposed to $\delta_S$ in the small $\delta$ case).

\subsection{Comparisons between the saddle-point solver and the \texorpdfstring{$\eta$}{η}-trick solver}\label{app:solver_comp}
Finally, we compare the saddle-point solver and the $\eta$-trick solver as introduced in \Cref{efficient_solvers} with details in  Appendix~\ref{efficient_solvers_appendix}.

\begin{figure}[t]
\centering
\begin{subfigure}{.47\textwidth}
    \centering
    \includegraphics[width=1\textwidth]{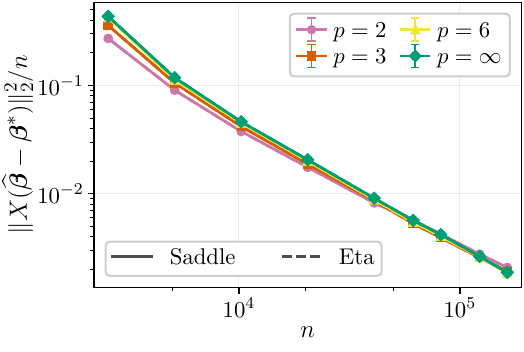}
\end{subfigure}
\quad
\centering
\begin{subfigure}{.47\textwidth}
    \centering
    \includegraphics[width=1\linewidth]{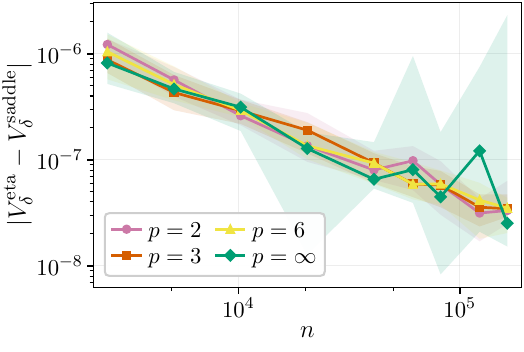}
\end{subfigure}
\caption{We plot the error of the saddle-point solver and the $\eta$-trick solver for the fast rate setup (left). We run 10 repeats and plot the mean with $\pm1$ standard deviation bands. We also plot the robust risk for the obtained estimate for both solvers (right). Here we plot the median over 10 repeats with 10-90 percentiles. We note that both output near-identical results.}
\label{fig_solvers_correct}
\end{figure}

\textbf{Correctness.} First, we run both solvers on the fast rate setup as detailed in \Cref{numerical_experiments}, with results shown in Figure \ref{fig_solvers_correct}. We see that both solvers achieve near-identical results (left plot). Indeed, the difference in robust risk error is $10^{-6}$ to $10^{-8}$ (right plot). This provide numerical support that both solvers are correctly solving the given optimization problem.

\begin{figure}[t]
\centering
\begin{subfigure}{.47\textwidth}
    \centering
    \includegraphics[width=1\linewidth]{figures/runtime.pdf}
\end{subfigure}
\quad
\begin{subfigure}{.47\textwidth}
    \centering
    \includegraphics[width=1\linewidth]{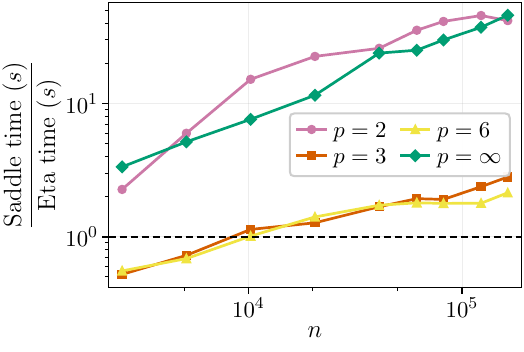}
\end{subfigure}
\caption{Computation time (left) and time ratio (right) of the saddle-point solver and the $\eta$-trick solver for the fast rate setup. We see that $\eta$-trick solver is, in general, faster, particularly for large $n$ and $p \in \{2,\infty\}$. Both methods are slower for $p \in \{3,6\}$, since the robust risk then lacks a closed-form~expression. Both plots show mean over 10 runs.}
\label{fig_solvers_speed}
\end{figure}

\textbf{Speed.} Next we compare the speed of the two solvers and how their computation times scale with $n$, using the fast rate setup. The computation time is shown in Figure \ref{fig_solvers_speed} (left). Here, we see that the $\eta$-trick solver is faster in general, being slightly faster for $p \in \{3,6\}$ (especially for large $n$) and several orders of magnitude faster for $p \in \{2,\infty\}$. We can also see this more clearly by looking at their time ratio, illustrated in Figure \ref{fig_solvers_speed} (right), where the region above the dashed line indicates that the $\eta$-trick solver is faster. The reason for the improved speed could be that the saddle-point formulation leads to slower convergence in general or that the DSP implementation is slower in particular; future work aims to find the underlying reason for this difference. Finally, we also note in Figure \ref{fig_solvers_speed} that it takes considerably more time to solve for $p \in \{3,6\}$ since the robust risk lacks a closed-form expression. Future work should look into other variations to improve the speed. In particular, using the scalar optimization form in \Cref{lemma_robust_risk_1D_formulation} could yield faster computations.

\end{document}